\documentclass[10pt,journal,compsoc]{IEEEtran}
\ifCLASSOPTIONcompsoc
  \usepackage[nocompress]{cite}
\else
  \usepackage{cite}
\fi
\ifCLASSINFOpdf
\else
\fi
\usepackage{amsmath, amsthm, amssymb}

\newtheorem{prop}{Proposition}
\usepackage{amsfonts, amssymb, amsmath, mathrsfs, bbm}
\usepackage[vlined, ruled]{algorithm2e}
\usepackage{setspace}
\usepackage[noend]{algpseudocode}
\usepackage{graphicx}
\usepackage{ragged2e} 
\usepackage{booktabs,makecell, multirow, multicol, tabularx}
\usepackage[pagebackref,breaklinks,colorlinks]{hyperref}
\usepackage{subfigure}

\usepackage{longtable}
\usepackage{float}

\begin{document}
%
\title{Dataset Distillation: A Comprehensive Review}

\author{Ruonan Yu$^*$,
        Songhua Liu$^*$,
        Xinchao Wang$^\dag$
\IEEEcompsocitemizethanks{\IEEEcompsocthanksitem Authors are  with the National University of Singapore.
\IEEEcompsocthanksitem $*$: Equal Contribution.%
\IEEEcompsocthanksitem $\dag$: Corresponding Author (xinchao@nus.edu.sg).}
}

%
%

\markboth{SUBMISSION TO IEEE TRANSACTIONS ON PATTERN ANALYSIS AND MACHINE INTELLIGENCE}%
{Yu \MakeLowercase{\textit{et al.}}: Dataset Distillation: A Comprehensive Review}
%



\IEEEtitleabstractindextext{%
\begin{abstract}
Recent success of deep learning is largely attributed to 
the sheer amount of data used for training deep neural networks.
Despite the unprecedented success, 
the massive data, unfortunately, significantly 
increases the burden on storage and transmission 
and further gives rise to a cumbersome model training process.
Besides, relying on the raw data for training
\emph{per se} yields concerns about privacy and copyright. 
To alleviate these shortcomings, 
dataset distillation~(DD), also known as dataset condensation (DC), 
was introduced and has recently attracted much research attention
in the community.
Given an original dataset, DD aims to derive a much smaller dataset containing  
synthetic samples, based on which the trained models
yield  performance comparable with those trained on the original dataset. 
In this paper, we give a comprehensive review and summary 
of recent advances in DD and its application. 
We first introduce the task formally and propose an overall algorithmic framework followed by all existing DD methods.
Next, we provide a systematic taxonomy of current methodologies in this area,
and discuss their theoretical interconnections.
We also present current challenges in DD through extensive experiments 
and envision possible directions for future works. 
\end{abstract}

}

\maketitle

\IEEEdisplaynontitleabstractindextext

%
\IEEEpeerreviewmaketitle

\IEEEraisesectionheading{\section{Introduction}\label{sec:introduction}}

%
%
%
%
\IEEEPARstart{D}{eep} 
learning has gained tremendous success during the past few years in various fields, such as computer vision\cite{krizhevsky2017imagenet}, natural language processing\cite{devlin2018bert}, and speech recognition\cite{amodei2016deep}. 
The groundbreaking deep models, such as 
AlexNet\cite{krizhevsky2017imagenet} in 2012, ResNet\cite{he2016deep} in 2016, Bert\cite{devlin2018bert} in 2018, alongside the more recent ViT\cite{dosovitskiy2020image}, CLIP\cite{radford2021learning}, and DALLE\cite{ramesh2022hierarchical}, 
all rely on large-scale datasets for training.
However, it  takes a great effort to deal with such a sheer amount of data for collecting, storing, transmitting, pre-processing, among others.
Moreover, training over massive datasets typically requires enormous computation costs, and sometimes thousands of GPU hours to achieve satisfactory performance, which, in turn, incommodes many applications that rely on training over datasets multiple times, such as hyper-parameter optimization~\cite{chen2022bidirectional,maclaurin2015gradient,lorraine2020optimizing} and neural architecture search~\cite{such2020generative,li2020random,elsken2019neural}. 
Even worse, information and data are growing explosively in the real world:
on the one hand, only training on new-coming data is prone to the catastrophic forgetting problem~\cite{goodfellow2013empirical,rebuffi2017icarl}, which hurts the performance significantly;
on the other hand, storing all historical data is highly cumbersome, if not infeasible at all. 
In summary, there are contradictions between the demands for high-accuracy models and the limited computation and storage resources. 
In order to solve the aforementioned problem, one natural idea is to compress original datasets into smaller ones and only store useful information for target tasks, which alleviates the burden on storage while preserving the model performance. 

A relatively straightforward method to obtain such smaller datasets is to select the most representative or valuable samples from original datasets so that models trained on these subsets can achieve as good performance as the original ones. 
This line of method is known as core-set or instance selection. 
Although effective, such heuristic selection-based methods discard a large fraction of training samples directly, dismissing their contribution to training results and often resulting in sub-optimal performance. 
Moreover, it inevitably raises concerns about privacy and copyright, if datasets containing raw samples are published and accessed directly without copyright or privacy protection~\cite{dong2022privacy,shokri2015privacy}. 

\begin{figure}[t]
\centering
\includegraphics[width=\linewidth]{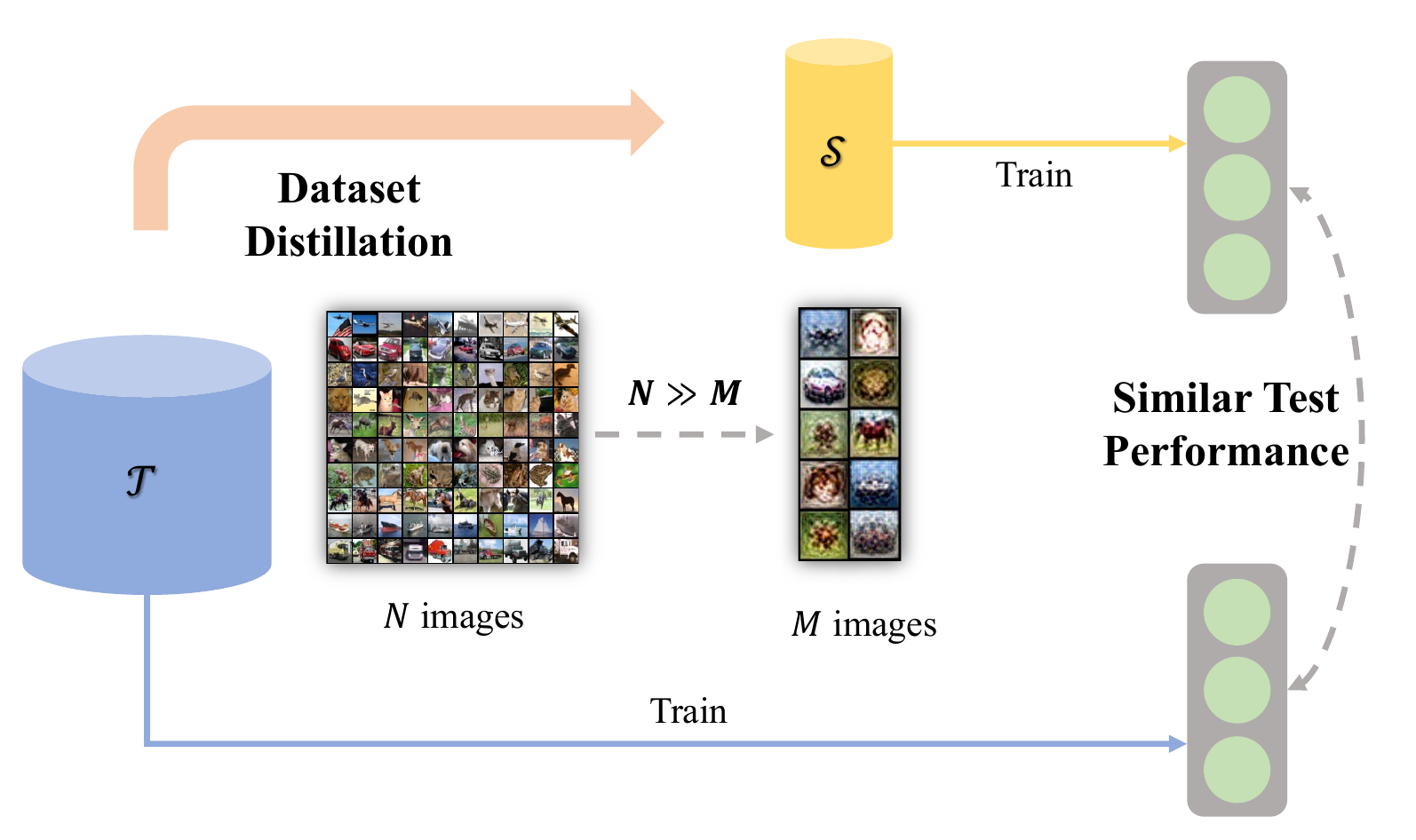} 

\caption{An overview for dataset distillation. Dataset distillation aims to generate a small informative dataset such that the models trained on these samples have similar test performance to those trained on the original dataset.}
\vspace{-0.5cm}
\label{fig:dc}
\end{figure}

To overcome the privacy and copyright issues, 
another line of work that focuses on 
generating new training data for compression
has been proposed, known as
dataset distillation (DD) or dataset condensation (DC).
We show the overall pipeline of DD in Fig.~\ref{fig:dc}.
Unlike the core-set fashion of directly selecting informative
training samples, methods along this line, 
which this paper will review, 
aim at synthesizing original datasets into a limited number of samples
such that they are learned or optimized to represent the knowledge of original datasets. 
It was first introduced and studied by Wang \textit{et al.}~\cite{wang2018dataset}, where an instructive method is proposed to update synthetic samples iteratively,
so that models trained on these samples can perform well on the real ones. 
This seminal work has attracted many follow-up studies in recent years. 
On the one hand, considerable progress has been made in improving the performance of DD by a series of methods~\cite{zhou2022dataset,nguyen2020dataset,nguyen2021dataset,zhao2021dataset,zhao2021datasetb,DBLP:journals/corr/abs-2110-04181},
and the real-world performances of models trained on synthetic datasets 
are indeed approaching
those trained on original ones as much as possible. 
On the other hand, a variety of works have extended the application of DD into a variety of research fields, such as continual learning~\cite{liu2020mnemonics,rosasco2022distilled,sangermano2022sample,wiewel2021condensed,masarczyk2020reducing} and federated learning~\cite{goetz2020federated,zhou2020distilled,xiong2022feddm,song2022federated,liu2022meta,hu2022fedsynth,pi2022dynafed}. 

This paper aims to provide an overview of current research in dataset distillation. 
Specifically, the contributions are listed as follows:
\begin{itemize}
    \item We explore and summarize existing works in dataset distillation and its application comprehensively. 
    \item We provide a systematic classification of current methods in DD. Categorized by the optimization objective, there are three mainstream solutions: performance matching, parameter matching, and distribution matching. Their relationship will also be discussed. 
    \item By abstracting all the key components in DD, we construct an overall algorithmic framework followed by all existing DD methods. 
    \item We discuss current challenges in DD and envision possible future directions for improvements. 
\end{itemize}

The rest of the article is organized as follows. 
We will begin with discussions of closely-related works for DD in Section \ref{sec:begin}, including core-set methods, hyper-parameter optimization, and generative models. 
Then in Section \ref{sec:classification}, we first define the dataset distillation problem formally, along with a general framework of the dataset distillation method. 
Based on this, various methods and their relationship will be systematically introduced. 
Section \ref{sec:application} will introduce applications of DD in various research fields. 
Section \ref{sec:experiments} will evaluate the performance of representation DD methods.
Section \ref{sec:challenges} will discuss the current challenges in the field and propose some possible directions for future works,
and Section \ref{sec:conclusion} will conclude the paper.

\section{Related Works}\label{sec:begin}
Dataset distillation (DD) aims distilling knowledge of the dataset into some synthetic data while preserving the performance of the models trained on it. 
Its motivations, objectives, and solutions are closely related to some existing fields, like knowledge distillation, core-set selection, generative modeling, and hyperparameter optimization. 
Their relationship with dataset distillation will be discussed in this section. 

\subsection{Knowledge Distillation}
Knowledge distillation (KD)~\cite{hinton2015distilling,romero2014fitnets,zagoruyko2016paying,gou2021knowledge} aims to transfer knowledge from a large teacher network to a smaller student network, such that the student network can preserve the performance of the teacher with reduced computational overhead. 
The seminal work by Hinton \textit{et al.}~\cite{hinton2015distilling} leads the student to mimic the outputs of the teacher, which can represent knowledge acquired by the teacher network. 
Afterward, improvements of KD have focused on four aspects: representations of knowledge, teacher-student architectures, distillation algorithms, and distillation schemes. 
First, knowledge can be represented by model response/output~\cite{hinton2015distilling,mirzadeh2020improved}, features~\cite{romero2014fitnets,xu2020feature,wang2020exclusivity}, and relation~\cite{you2017learning,park2019relational,liu2019knowledge}. 
Second, teacher-student architectures refer to the network architectures of teacher and student models, which determines the quality of knowledge acquisition and distillation from teacher to student~\cite{gou2021knowledge}. 
Third, distillation algorithms determine the ways of knowledge transfer. 
A simple and typical way is to match the knowledge captured by the teacher and student models directly~\cite{hinton2015distilling,romero2014fitnets}. 
Beyond that, many different algorithms are proposed to handle more complex settings, such as adversarial distillation~\cite{wang2018kdgan}, attention-based distillation~\cite{zagoruyko2016paying}, and data-free distillation~\cite{lopes2017data,fang2022up}. 
Finally, distillation schemes control training configurations of teacher and student, and there are offline-~\cite{hinton2015distilling,romero2014fitnets}, online-~\cite{zhang2018shufflenet}, and self-distillation~\cite{zhang2019your}. 
As for application, KD is widely used in ensemble learning~\cite{radosavovic2018data} and model compression~\cite{ba2014deep,romero2014fitnets,howard2017mobilenets}.  

The concept of DD is inspired by KD~\cite{wang2018dataset}. 
Specifically, DD aims at a lightweight dataset, while KD aims at a lightweight model. 
In this view, DD and KD are only conceptually related but technically orthogonal. 
It is worth noting that, similar to DD, recent data-free KD methods~\cite{lopes2017data,nayak2019zero,fang2022up} are also concerned with the generation of synthetic training samples since original training datasets are unavailable. 
Their differences are two-fold. 
On the one hand, data-free KD takes a teacher model as input, while DD takes an original dataset as input. 
On the other hand, data-free KD aims at a student model with performance similar to the teacher. 
In order for a satisfactory student, it generally relies on a large number of synthetic samples. 
For DD, it aims at a smaller dataset, whose size is pre-defined and typically small due to limited storage budgets. 

\subsection{Core-set or Instance Selection}
core-set or instance selection is a classic selection-based method to reduce the size of the training set. 
It only preserves a subset of the original training dataset containing only valuable or representative samples, such that the models trained on it can achieve similar performance to those trained on the original. 
Since simply searching for the subset with the best performance is NP-Hard, current core-set methods select samples mainly based on some heuristic strategies. 
Some early methods generally expect a consistent data distribution between core-sets and the original ones~\cite{welling2009herding,chen2012super,feldman2011scalable,bachem2015coresets}. 
For example, Herding~\cite{welling2009herding,chen2012super}, one of the classic core-set methods, aims to minimize the distance in feature space between centers of the selected subset and the original dataset by incrementally and greedily adding one sample each time. 
In recent years, more strategies beyond matching data distribution have been proposed. 
For instance, bi-level optimization-based methods are proposed and applied in fields of continual learning~\cite{borsos2020coresets}, semi-supervised learning~\cite{killamsetty2021retrieve}, and active learning~\cite{killamsetty2021glister}. 
For another example, Craig \textit{et al.}~\cite{mirzasoleiman2020coresets} propose to find the core-set with the minimal difference between gradients produced by the core-set and the original core-set with respective to the same neural network. 

Aiming to synthesize condensed samples, many objectives in DD shares similar spirits as core-set techniques. 
As a matter of fact, on the functional level, any valid objective in core-set selection is also applicable in DD, such as those matching distribution, bi-level optimization, and matching gradient-based methods. 
However, core-set and DD have distinct differences. 
On the one hand, core-set is based on selected samples from original datasets, while DD is based on synthesizing new data. 
In other words, the former contains raw samples in real datasets, while the latter does not require synthetic samples to look real, potentially providing better privacy protection. 
On the other hand, due to the NP-hard nature, core-set methods typically rely on heuristics criteria or greedy strategies to achieve a trade-off between efficiency and performance. 
Thus, they are more prone to sub-optimal results compared with DD, where all samples in condensed datasets are learnable. 

\subsection{Generative Models}
Generative models are capable of generating new data by learning the latent distribution of a given dataset and sampling from it. 
The primary target is to generate realistic data, like images, text, and video, which can fool human beings. 
DD is related to this field in two aspects. 
For one thing, one classic optimization objective in DD is to match the distribution of synthetic datasets and original ones, which is conceptually similar to generative modeling, \textit{e.g.}, generative adversarial network (GAN)~\cite{goodfellow2020generative,mirza2014conditional} and variational autoencoders (VAE)~\cite{kingma2013auto}. 
Nevertheless, generative models try to synthesize realistic images, while DD cares about generating informative images that can boost training efficiency and save storage costs. 

For another, generative models can also be effective in DD. 
In DD, synthetic samples are not learned directly in some works. 
Instead, they are produced by some generative models. 
For example, IT-GAN~\cite{zhao2022synthesizing} adopts conditional GANs~\cite{mirza2014conditional} and GANs inversion~\cite{abdal2019image2stylegan,zhu2020domain} for synthetic data generation so that it only needs to update latent codes rather than directly optimizing image pixels; KFS~\cite{lee2022dataset} uses latent codes together with several decoders to output synthetic data. 
This usage of generative models is known as a kind of synthetic dataset parameterization, whose motivation is that latent codes provide a more compact representation of knowledge than raw samples and thus improve storage efficiency. 

\subsection{Hyperparameter Optimization}
The performance of machine learning models is typically sensitive to the choice of hyperparameters~\cite{maclaurin2015gradient,lorraine2020optimizing}. 
Researches in hyperparameter optimization aim to automatically tune hyperparameters to yield optimal performance for machine learning models, eliminating the tedious work of manual adjustment and thus improving efficiency and performance. 
Hyperparameter optimization has a rich history. 
Early works, like \cite{snoek2012practical} and \cite{bergstra2013making}, are gradient-free model-based methods, choosing hyperparameters to optimize the validation loss after fully training the model. 
However, the efficiency of these methods is limited when dealing with tens of hyperparameters~\cite{maclaurin2015gradient}. 
Thus, gradient-based optimization methods~\cite{maclaurin2015gradient,franceschi2017forward,franceschi2018bilevel,lorraine2020optimizing} are proposed and tremendously succeed in scaling up, which open up a new way of optimizing high-dimensional hyperparameters. 
Introducing gradients into hyperparameter optimization problems opens up a new way of optimizing high-dimensional hyperparameters~\cite{lorraine2020optimizing}. 

DD can also be converted into a hyperparameter optimization problem if each sample in synthetic datasets is regarded as a high-dimensional hyperparameter. 
In this spirit, the seminal algorithm proposed by Wang \textit{et al.}~\cite{wang2018dataset} optimizes the loss on real data for models trained by synthetic datasets with gradient decent, shares essentially the same idea with gradient-based hyperparameter optimization. 
However, the focus and final goals of DD and hyperparameter optimization are different. 
The former focus more on getting synthetic datasets to improve the efficiency of the training process, while the latter tries tuning hyperparameter to optimize model performance. 
Given such orthogonal objectives, other streams of DD methods, like parameter matching and distribution matching, can hardly be related to hyperparameter optimization.  

\section{Dataset Distillation Methods}\label{sec:classification}

Dataset distillation (DD), also known as dataset condensation (DC), is first formally proposed by Wang \textit{et al.}~\cite{wang2018dataset}. 
The target is to extract knowledge from a large-scale dataset and build a much smaller synthetic dataset, such that models trained on it have comparable performance to those trained on the original dataset. 
This section first provides a comprehensive definition of dataset distillation and summarizes a general workflow for DD. 
Then, we categorize current DD methods by different optimization objectives: performance matching, parameter matching, and distribution matching. 
Algorithms associated with these objectives will be illustrated in detail, and their potential relationships will also be discussed. 
Some works also focus on synthetic data parameterization and label distillation, which will be introduced in the following parts. 

\subsection{Definition of Dataset Distillation}
The canonical dataset distillation problem involves learning a small set of synthetic data from an original large-scale dataset so that models trained on the synthetic dataset can perform comparably to those trained on the original. 
Given a real dataset consisting of $|\mathcal{T}|$ pairs of training images and corresponding labels, denoted as 
$\mathcal{T} = (X_t, Y_t)$, where $X_t \in \mathbb{R}^{N \times d}$, $N$ is the number of real samples, $d$ is the number of features, $Y_t \in \mathbb{R}^{N\times C}$, and $C$ is the number of output entries, 
the synthetic dataset is denoted as $\mathcal{S}=(X_s, Y_s)$, where $X_s \in \mathbb{R}^{M\times D}$, $M$ is the number of synthetic samples, $Y_s \in \mathbb{R}^{M\times C}$, $M \ll N$, and $D$ is the number of features for each sample. 
For the typical image classification tasks, $D$ is $height \times width \times channels$, and $y$ is a one-hot vector whose dimension $C$ is the number of classes. 
Formally, we formulate the problem as the following:
\begin{equation}\small
    \text{\small $\mathcal{S} = \mathop{\arg\min}\limits_{\mathcal{S}}\mathcal{L}(\mathcal{S}, \mathcal{T})$}, \label{eq:def}
\end{equation}
where $\mathcal{L}$ is some objective for dataset distillation, which will be elaborated in the following contents. 

\subsection{General Workflow of Dataset Distillation}\label{sec:general}
In this section, we illustrate the general workflow for current DD methods. 
Given that the primary goal of DD is to derive a smaller synthetic dataset given a real dataset such that neural networks trained on the synthetic data can have comparable performance with those trained on the real one, the objective function $\mathcal{L}$ in Eq.~\ref{eq:def} would rely on some neural networks. 
Therefore, two key steps in current DD methods are to train neural networks and compute $\mathcal{L}$ via these networks. 
They perform alternately in an iterative loop to optimize synthetic datasets $\mathcal{S}$, which formulates a general workflow for DD as shown in Alg.~\ref{framework}. 

Firstly, $\mathcal{S}$ is initialized before the optimization loop, which may have crucial impacts on the convergence and final performance of condensation methods~\cite{cui2022dc}. 
Typically, the synthetic dataset $\mathcal{S}$ is usually initialized in two ways: randomly initialization \textit{e.g.}, from Gaussian noise, and randomly selected real samples from the original dataset $\mathcal{T}$. 
The latter is adopted in most DD methods currently. 
Moreover, some coreset methods, \textit{e.g.}, K-Center~\cite{sener2017active}, can also be applied to initialize the synthetic dataset~\cite{cui2022dc}. 

In the iterative loop for synthetic dataset optimization, neural networks and synthetic data are updated alternately. 
Usually, to avoid overfitting problems and make the output synthetic dataset more generalizable, the network $\theta$ will be periodically fetched at the beginning of the loop. 
The networks can be randomly initialized or loaded from some cached checkpoints. 
For the former, common methods for initializing neural networks are applicable, \textit{e.g.}, Kaiming~\cite{he2015delving}, Xavier~\cite{glorot2010understanding}, and normal distribution. 
For the latter, cached checkpoints can be training snapshots, \textit{i.e.}, updated networks in previous iterations~\cite{zhao2021dataset,zhao2021datasetb,lee2022contrastive,jiang2022delving,zhou2022dataset}, or pretrained checkpoints for different epochs on the original dataset~\cite{cazenavette2022dataset}. 
The fetched network is updated via $\mathcal{S}$~\cite{zhao2021dataset,zhao2021datasetb} or $\mathcal{T}$~\cite{kim2022dataset} for some steps if needed. 
Then, the synthetic dataset is updated through some dataset distillation objective $\mathcal{L}$ based on the network, which will be introduced in detail in Section~\ref{sec:objectives}. 
Here, the order of network update and synthetic dataset update may be various in some works~\cite{zhao2021dataset,zhao2021datasetb,lee2022contrastive,jiang2022delving,kim2022dataset,zhou2022dataset}. 
For details on updating networks and synthetic datasets, please refer to Fig.~\ref{fig:taxonomy}. 

\begin{algorithm}[t]
\SetAlgoLined
\setstretch{1.35}
\caption{Dataset Distillation Framework}
\KwIn{Original dataset $\mathcal{T}$} 
\KwOut{Synthetic dataset $\mathcal{S}$}
\label{framework}
\BlankLine
Initialize $\mathcal{S}$ \Comment{Random, real, or core-set}

\While{not converge}{
Get a network $\theta$ \Comment{Random or from some cache}

Update $\theta$ and cache it if necessary \\\Comment{Via $\mathcal{S}$ or $\mathcal{T}$, for some steps}

Update $\mathcal{S}$ via $\mathcal{L}(\mathcal{S},\mathcal{T})$\\\Comment{PerM, ParM, DisM, or their variants}
}
\Return{$\mathcal{S}$}
\end{algorithm}

\subsection{Optimization Objectives in DD}\label{sec:objectives}
To obtain synthetic datasets that are valid to replace original ones in downstream training, different works in DD propose different optimization objectives, denoted as $\mathcal{L}$ in Eq.~\ref{eq:def}. 
This part will introduce three mainstream solutions: performance matching, parameter matching, and distribution matching. 
We will also reveal the relationships between them.

\subsubsection{Performance Matching}

\textbf{Meta Learning Based Methods.}
Performance matching was first proposed in the seminal work by Wang \textit{et al.}~\cite{wang2018dataset}. 
This approach aims to optimize a synthetic dataset such that neural networks trained on it could have the lowest loss on the original dataset, and thus the performance of models trained by synthetic and real datasets is matched:
\begin{equation}\small
\begin{split}
\text{\small $\mathcal{L}(\mathcal{S},\mathcal{T})$}
    &\text{\small $=\mathbb{E}_{\theta^{(0)} \sim \Theta}[l(\mathcal{T}; \theta^{(T)})]$}, \\
    \text{\small $\theta^{(t)}$}
    &\text{\small $=\theta^{(t-1)}-\eta\nabla l(\mathcal{S}; \theta^{(t-1)})$},\\
\end{split}
\label{eq:per}
\end{equation}
where $\Theta$ is the distribution for initialization of network parameters, $l$ is the loss function to train networks, \textit{e.g.}, cross-entropy loss for classification, $T$ denotes the number of inner iterations, and $\eta$ is the learning rate for inner loop. 
The objective of performance matching indicates a bi-level optimization algorithm: in inner loops, weights of a differentiable model with parameter $\theta$ are updated with $\mathcal{S}$ via gradient decent, and the recursive computation graph is cached; in outer loops, models trained after inner loops are validated on $\mathcal{T}$ and the validation loss is backpropagated through the unrolled computation graph to $\mathcal{S}$. 
DD methods based on such an objective are similar to bi-level meta-learning~\cite {DBLP:conf/icml/FinnAL17} and gradient-based hyperparameter optimization techniques~\cite{lorraine2020optimizing}. 
The intuition is shown in Fig.~\ref{fig:perm}. 

The following work by Deng \textit{et al.}~\cite{deng2022remember} finds that equipping the update of parameter $\theta$ for the inner loop with momentum $m$ can improve performance significantly. 
This alternative update function can be written as:
\begin{equation}\small
\begin{split}
    \text{\small $m^{(0)}$}&\text{\small $=0$},\\
    \text{\small $m^{(t)}$}&\text{\small $=\beta m^{(t-1)}+\nabla l(\mathcal{S}; \theta^{(t-1)})$},\\
    \text{\small $\theta^{t}$}&\text{\small $=\theta_{t-1}-\eta m^{(t)}$},
\end{split}
\label{eq:momentum}
\end{equation}
where $\beta$ is a hyperparameter indicating the momentum rate.

\textbf{Kernel Ridge Regression Based Methods.} 
The above meta-learning-based method involves bi-level optimization to compute the gradient of validation loss \textit{w.r.t} synthetic datasets by backpropagating through the entire training graph. 
In this way, outer optimization steps are computationally expensive, and the GPU memory required is proportional to the number of inner loops. 
Thus, the number of inner loops is limited, which results in insufficient inner optimization and bottlenecks the performance~\cite{zhao2021dataset}. 
It is also inconvenient for this routine to be scaled up to large models. 
There is a class of methods tackling this problem~\cite{nguyen2020dataset,nguyen2021dataset,zhou2022dataset,loo2022efficient} based on kernel ridge regression (KRR), which performs convex optimization and results in a closed-form solution for the linear model which avoids extensive inner-loop training. 
Projecting samples into a high-dimensional feature space with a non-linear neural network $f$ parameterized by $\theta$, where $\theta$ is sampled from some distribution $\Theta$, the performance matching metric can be represented as: 
\begin{equation}\small
\begin{split}
    \text{\small $\mathcal{L}(\mathcal{S},\mathcal{T})$} &\text{\small $= \mathbb{E}_{\theta\sim\Theta}[||Y_t - f_\theta(X_t)W_{\mathcal{S},\theta}^*||^2]$},\\
    \text{\small $W_{\mathcal{S},\theta}^*$}&\text{\small $= \mathop{\arg\min}\limits_{W_{\mathcal{S},\theta}}\{||Y_s-f_\theta(X_s)W_{\mathcal{S},\theta}||^2+\lambda||W_{\mathcal{S},\theta}||^2\}$}\\
    &\text{\small $=f_\theta(X_s)^T(f_\theta(X_s)f_\theta(X_s)^T+\lambda I)^{-1}Y_s$}.
\end{split}
\label{eq:linear}
\end{equation}
Here $\lambda$ is a small number for regularization. 
Rewriting Eq.~\ref{eq:linear} in kernel version, we have:
\begin{equation}\small
\begin{split}
    \text{\small $\mathcal{L}(\mathcal{S},\mathcal{T})=\mathbb{E}_{\theta\sim\Theta}[||Y_t - K_{X_tX_s}^\theta(K_{X_sX_s}^\theta+\lambda I)^{-1}Y_s||^2]$},
\end{split}
\label{eq:kernel}
\end{equation}
where $K_{X_1X_2}^\theta=f_{\theta}(X_1){f_{\theta}(X_2)}^T$. 
Nguyen \textit{et al.}~\cite{nguyen2020dataset,nguyen2021dataset} propose to perform KRR with the neural tangent kernel (NTK) instead of training neural networks for multiple steps to learn synthetic datasets, based on the property that KRR for NTK approximates the training of the corresponding wide neural network~\cite{jacot2018neural,lee2019wide,arora2019exact,lee2020finite}. 

To alleviate the high complexity of computing NTK, Loo \textit{et al.}~\cite{loo2022efficient} propose RFAD, based on Empirical Neural Network Gaussian Process~(NNGP) kernel~\cite{neal2012bayesian,lee2017deep} alternatively. 
They also adopt platt scaling~\cite{platt1999probabilistic} by applying cross-entropy loss to labels of real data instead of mean square error, which further improves the performance and is demonstrated to be more suitable for classification tasks. 

\begin{figure}[t]
\centering
\includegraphics[width=\linewidth]{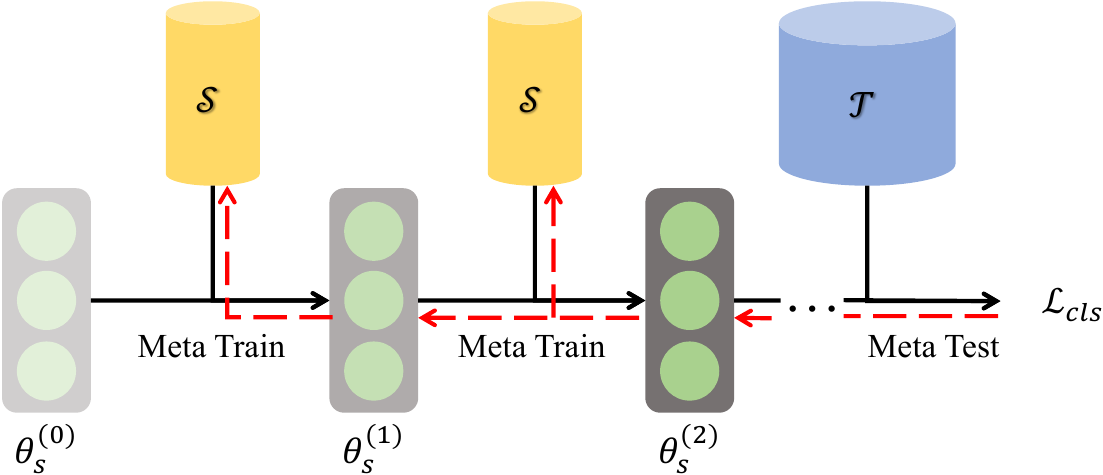} 
\vspace{-0.4cm}
\caption{Meta learning based performance matching. The bi-level learning framework aims to optimize the meta test loss on the real dataset $\mathcal{T}$, for the model meta-trained on the synthetic dataset $\mathcal{S}$.}
\vspace{-0.5cm}
\label{fig:perm}
\end{figure}

Concurrently, FRePo~\cite{zhou2022dataset} decomposes a neural network into a feature extractor and a linear classifier, and considers $\theta$ and $W_{\mathcal{S},\theta}$ in Eq.~\ref{eq:linear} as parameters of the feature extractor and the linear classifier respectively. 
Instead of pursuing a fully-optimized network in an entire inner training loop, it only obtains optimal parameters for the linear classifier via Eq.~\ref{eq:linear}, and the feature extractor is trained on the current synthetic dataset, which results in a decomposed two-phase algorithm. 
Formally, its optimization objective can be written as:
\begin{equation}\small
\begin{split}
    \text{\small $\mathcal{L}(\mathcal{S},\mathcal{T})$}&\text{\small $=\mathbb{E}_{\theta^{(0)}\sim\Theta}[\sum_{t=0}^T||Y_t - K_{X_tX_s}^{\theta^{(t)}}(K_{X_sX_s}^{\theta^{(t)}}+\lambda I)^{-1}Y_s||^2]$},\\
    \text{\small $\theta_{\mathcal{S}}^{(t)}$}&\text{\small $=\theta_{\mathcal{S}}^{(t-1)}-\eta\nabla l(\mathcal{S}; \theta_{\mathcal{S}}^{(t-1)})$}.\\
\end{split}
\label{eq:frepo}
\end{equation}
For FRePo, updates of synthetic data and networks are decomposed, unlike the meta-learning-based method in Eq.~\ref{eq:per}, which requires backpropagation through multiple updates. 

\begin{figure*}[t]
  \centering
  \subfigure[]{
 \includegraphics[width=0.3\linewidth]{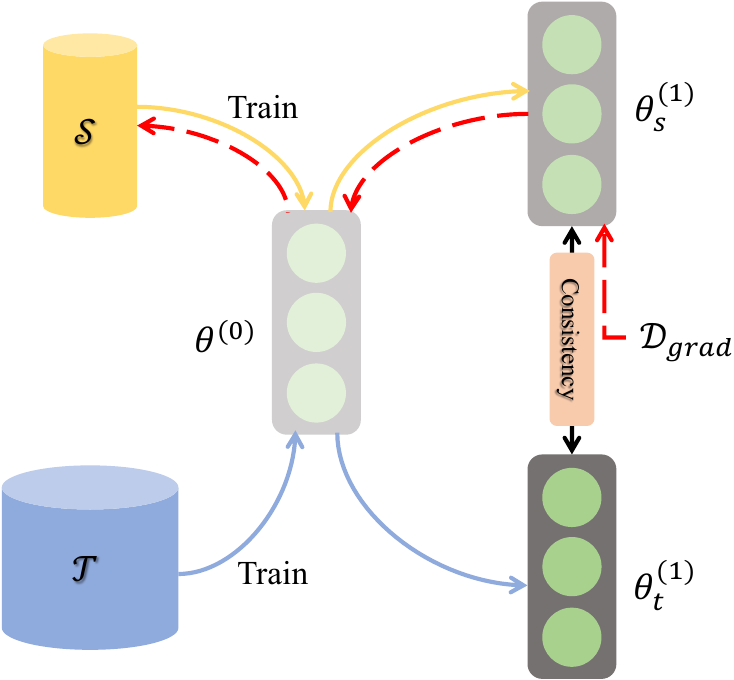}
  }
  \subfigure[]{
  \includegraphics[width=0.62\linewidth]{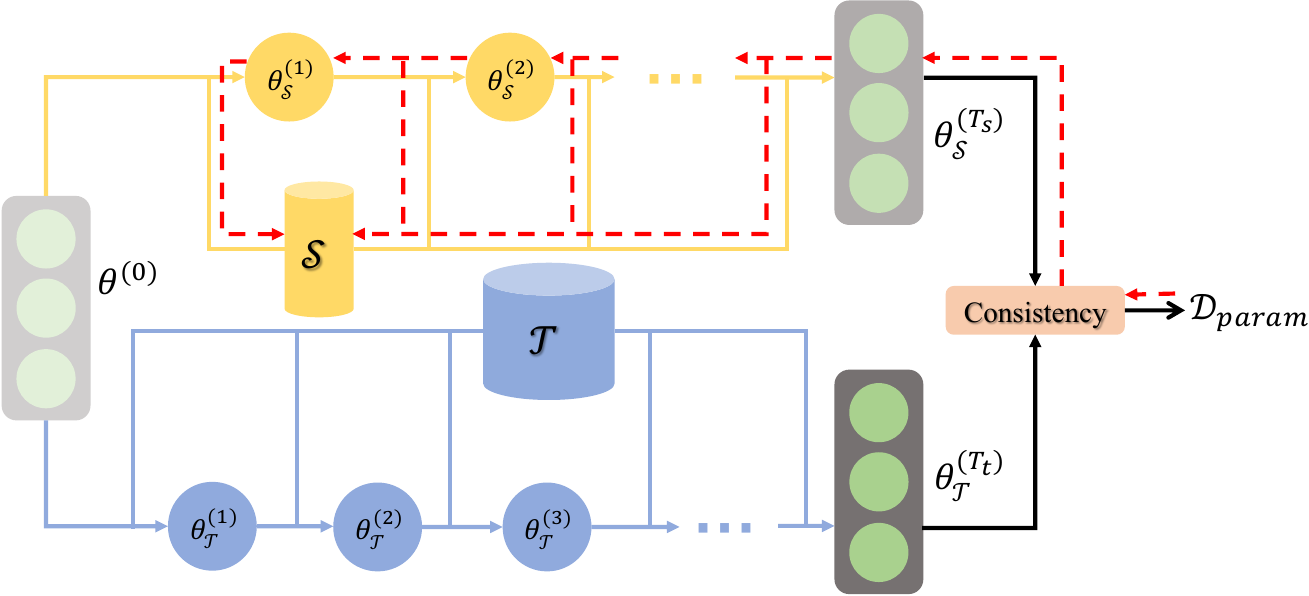} 
  }
  \vspace{-0.2cm}
  \caption{(a) Single-step parameter matching. (b) Multi-step parameter matching. They optimize the consistency of trained model parameters using real data and synthetic data. For single-step parameter matching, it is equivalent to matching gradients. For multi-step parameter matching, it is also known as matching training trajectories.}
  \vspace{-0.5cm}
  \label{fig:param}
\end{figure*}

\subsubsection{Parameter Matching}
The approach of matching parameters of neural networks in DD is first proposed by Zhao \textit{et al.}~\cite{zhao2021dataset}; and is extended by a series of following works~\cite{lee2022contrastive,jiang2022delving,loo2022efficient,cazenavette2022dataset}. 
Unlike performance matching, which optimizes the performance of networks trained on synthetic datasets, the key idea of parameter matching is to train the same network using synthetic datasets and original datasets for some steps, respectively, and encourage the consistency of their trained neural parameters. 
According to the number of training steps using $\mathcal{S}$ and $\mathcal{T}$, parameter matching methods can be further divided into two streams: single-step parameter matching and multi-step parameter matching. 

\textbf{Single-Step Parameter Matching.} 
In single-step parameter matching, as shown in Fig.~\ref{fig:param}(a), a network is updated using $\mathcal{S}$ and $\mathcal{T}$ for only $1$ step, respectively, and their resultant gradients with respective to $\theta$ are encouraged to be consistent, which is also known as gradient matching. 
It is first proposed by Zhao \textit{et al.}~\cite{zhao2021dataset} and is extended by a series of following works~\cite{lee2022contrastive,jiang2022delving,kim2022dataset,zhang2022accelerating}. 
After each step of updating synthetic data, the network used for computing gradients is trained on $\mathcal{S}$ for $T$ steps. 
In this case, the objective function can be formalized as follows:
\begin{equation}\small
\begin{split}
    \text{\small $\mathcal{L}(\mathcal{S},\mathcal{T})$}&\text{\small $=\mathbb{E}_{\theta^{(0)} \sim \Theta}[\sum_{t=0}^T\mathcal{D}(\mathcal{S},\mathcal{T};
    \theta^{(t)})]$},\\
    \text{\small $\theta^{(t)}$} &\text{\small $= \theta^{(t-1)}-\eta\nabla l(\mathcal{S};\theta^{(t-1)})$},
\end{split}
\end{equation}
where metric $\mathcal{D}$ measures the distance between gradients $\nabla l(\mathcal{S};\theta^{(t)})$ and $\nabla l(\mathcal{T};\theta^{(t)})$. 
Since only a single-step gradient is necessary, and updates of synthetic data and networks are decomposed, this approach is memory-efficient compared with meta-learning-based performance matching. 
Particularly, in image classification tasks, when updating the synthetic dataset $\mathcal{S}$, Zhao \textit{et al.}~\cite{zhao2021dataset} sample each synthetic and real batch pair $\mathcal{S}_c$ and $\mathcal{T}_c$ from $\mathcal{S}$ and $\mathcal{T}$ respectively containing samples from the $c$-th class, and each class of synthetic data is updated separately in each iteration: 
\begin{equation}\small
\begin{split}
     \text{\small $\mathcal{D}(\mathcal{S},\mathcal{T};
    \theta)$}&\text{\small $=\sum_{c=0}^{C-1}d(\nabla l(\mathcal{S}_c;\theta),\nabla l(\mathcal{T}_c;\theta))$},\\
     \text{\small $d(\mathbf{A},\mathbf{B})$}&\text{\small $=\sum_{i=1}^{L}\sum_{j=1}^{J_i}(1-\frac{\mathbf{A^{(i)}_j}\cdot\mathbf{B^{(i)}_j}}{\Vert\mathbf{A^{(i)}_j}\Vert\Vert\mathbf{B^{(i)}_j}\Vert})$},
\end{split}
\label{eq:grad}
\end{equation}
where $C$ is the total number of classes, $L$ is the number of layers in neural networks, $i$ is the layer index, $J_i$ is the number of output channels for the $i$-th layer, and $j$ is the channel index. 
As shown in Eq.~\ref{eq:grad}, the original idea proposed by Zhao \textit{et al.}~\cite{zhao2021dataset} adopts the negative cosine similarity to evaluate the distance between two gradients. 
It also preserves the layer-wise and channel-wise structure of the network to get an effective distance measurement. 

Nevertheless, this method has some limitations, \textit{e.g.}, the distance metric between two gradients considers each class independently and ignores relationships underlain for different classes. 
Thus, class-discriminative features are largely neglected. 
To remedy this issue, Lee \textit{et al.}~\cite{lee2022contrastive} propose a new distance metric for gradient matching considering class-discriminative features:
\begin{equation}\small
\begin{split}
    \text{\small $ \mathcal{D}(\mathcal{S},\mathcal{T};
    \theta)= d(\frac{1}{C}\sum_{c=0}^{C-1}\nabla l(\mathcal{S}_c;\theta), \frac{1}{C}\sum_{c=0}^{C-1}\nabla l(\mathcal{T}_c;\theta))$}. \\
\end{split}
     \label{eq:grad_improved}
\end{equation}
Compared with Eq.~\ref{eq:grad}, the position of summation over all classes is different. 

Similar insight is also discussed in Jiang \textit{et al.}~\cite{jiang2022delving}, which combines the two formulation in Eq.~\ref{eq:grad} and Eq.~\ref{eq:grad_improved} with a hyperparameter $\lambda$ for their balance:
\begin{equation}\small
\begin{split}
     \text{\small $\mathcal{D}(\mathcal{S},\mathcal{T};
    \theta)=$}&\text{\small $\sum_{c=0}^{C-1}d(\nabla l(\mathcal{S}_c;\theta),\nabla l(\mathcal{T}_c;\theta))$}\\
     &\text{\small $+\lambda d(\frac{1}{C}\sum_{c=0}^{C-1}\nabla l(\mathcal{S}_c;\theta), \frac{1}{C}\sum_{c=0}^{C-1}\nabla l(\mathcal{T}_c;\theta))$}. 
\end{split}
\end{equation}
Moreover, the work also argues that the cosine distance-based function $d$ in Eq.~\ref{eq:grad} only considers the angle between the gradients. 
Since synthetic datasets $\mathcal{S}$ are small in general, it is prone to over-fitting and bad generalization problems. 
Also, it is found that gradient norms of $\mathcal{S}$ degrade quickly and will stick in a local minimum, \textit{i.e.}, $\nabla l(\mathcal{S};\theta) = 0$ after only a few gradient descent steps. 
In this case, it is not very sensible to match the angle. 
Thus, Jiang \textit{et al.}~\cite{jiang2022delving} further consider the magnitude of the gradient vectors by adding the Euclidean distance between them:
\begin{equation}\small
     \text{\small $d(\mathbf{A}, \mathbf{B}) = \sum_{j=1}^{J}(1-\frac{\mathbf{A_j}\cdot\mathbf{B_j}}{\Vert\mathbf{A_j}\Vert\Vert\mathbf{B_j}\Vert} + \Vert\mathbf{A_j}-\mathbf{B_j}\Vert)$}.
\end{equation}

Beyond formulations of the distance between two groups of gradients, there are also following works focusing on more effective strategies to update networks. 
The original work by Zhao \textit{et al.}~\cite{zhao2021dataset} adopts synthetic dataset $\mathcal{S}$ to update network parameters $\theta$, while $\mathcal{S}$ and $\theta$ are strongly bonded in the optimization process, leading to a chicken-egg problem. 
Also, the size of the synthetic dataset is tiny, the overfitting problem easily occurs in the early stage of the training, and network gradients vanish quickly. 
To solve these problems, Kim \textit{et al.}~\cite{kim2022dataset} proposes to train networks on the real dataset $\mathcal{T}$. 
In this way, the overfitting problem can be alleviated due to the large size of $\mathcal{T}$, and network parameters are independent of $\mathcal{S}$. 

Besides, existing gradient matching methods~\cite{zhao2021dataset,zhao2021datasetb,kim2022dataset,lee2022contrastive,jiang2022delving} are computationally expensive to gain synthetic datasets with satisfactory generalizability, as thousands of differently initialized networks are required to optimize $\mathcal{S}$. 
To accelerate dataset distillation, Zhang \textit{et al.}~\cite{zhang2022accelerating} propose model augmentations, which adopt early-stage models to form a candidate model pool and apply weight perturbation on selected early-stage models from the pool during dataset distillation to increase model diversity. 
This way, synthetic datasets with comparable performance can be obtained with significantly fewer random networks and optimization iterations.

\textbf{Multi-Step Parameter Matching.} 
For single-step parameter matching, since only single-step gradient is matched, errors may be accumulated in evaluation where models are updated by synthetic data for multiple steps. 
To solve this problem, Cazenavette \textit{et al.}~\cite{cazenavette2022dataset} propose a multi-step parameter matching approach known as matching training trajectory (MTT). 
In this method, $\theta$ will be initialized and sampled from checkpoints of training trajectories on original datasets. 
As shown in Fig.~\ref{fig:param}(b), starting from $\theta^{(0)}$, the algorithm trains the model on synthetic datasets for $T_s$ steps and the original dataset for $T_t$ steps, respectively, where $T_s$ and $T_t$ are hyperparameters, and tries to minimize the distance of the endings of these two trajectories, \textit{i.e.}, $\theta_{\mathcal{S}}^{(T_s)}$ and $\theta_{\mathcal{T}}^{(T_t)}$:
\begin{equation}\small
\label{eq3}
\begin{split}
    \text{\small $\mathcal{L}(\mathcal{S},\mathcal{T})$}&\text{\small $=\mathbb{E}_{\theta^{(0)} \sim \Theta}[\mathcal{D}(\theta_{\mathcal{S}}^{(T_s)}, \theta_{\mathcal{T}}^{(T_t)})]$},\\
    \text{\small $\theta_{\mathcal{S}}^{(t)} $}&\text{\small $= \theta_{\mathcal{S}}^{(t-1)}-\eta\nabla l(\mathcal{S}; \theta_{\mathcal{S}}^{(t-1)})$},\\
    \text{\small $\theta_{\mathcal{T}}^{(t)}$} &\text{\small $= \theta_{\mathcal{T}}^{(t-1)}-\eta\nabla l(\mathcal{T}; \theta_{\mathcal{T}}^{(t-1)})$},\\
\end{split}
\end{equation}
where $\Theta$ is the distribution for cached teacher trajectories, $l$ is the loss function for training networks, and $\mathcal{D}$ is defined as follows:
\begin{equation}\small
    \mathcal{D}(\theta_{\mathcal{S}}^{(T_s)}, \theta_{\mathcal{T}}^{(T_t)})=\frac{\Vert\theta_{\mathcal{S}}^{(T_s)}-\theta_{\mathcal{T}}^{(T_t)})\Vert^2}{\Vert\theta_{\mathcal{T}}^{(T_t)}-\theta^{(0)}\Vert^2}. 
\end{equation}
Note that the total loss is normalized by the distance between the starting point $\theta^{(0)}$ and the expert endpoint $\theta_{\mathcal{T}}^{(T_t)}$.
This normalization helps get a strong signal where the expert does not move as much at later training epochs and self-calibrates the magnitude difference across neurons and layers. 
It is demonstrated that this multi-step parameter matching strategy yields better performance than the single-step counterpart. 

As a following work along this routine, Li \textit{et al.}~\cite{li2022datasetb} find that a few parameters are difficult to match in dataset distillation for multi-step parameter matching, which negatively affects the condensation performance. 
To remedy this problem and generate a more robust synthetic dataset, they adopt parameter pruning, removing the difficult-to-match parameters when the similarity of the model parameters trained on synthetic datasets $\theta_{\mathcal{S}}^{(T_s)}$ and on real datasets $\theta_{\mathcal{T}}^{(T_t)}$ is less than a threshold after updating networks in each dataset distillation iteration. 
The synthetic dataset is updated by optimizing the objective function $\mathcal{L}$ calculated with pruned parameters. 

\begin{figure}[t]
\centering
\includegraphics[width=\linewidth]{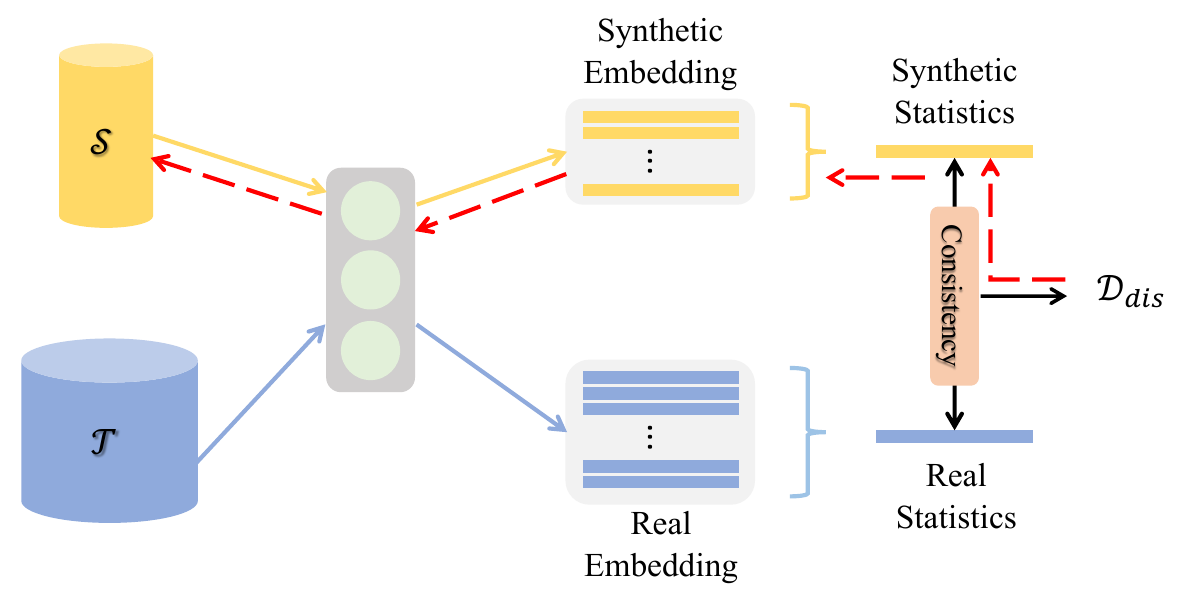} 
\vspace{-0.5cm}
\caption{Distribution matching. It matches statistics of features in some networks for synthetic and real datasets.}
\vspace{-0.5cm}
\label{fig:dism}
\end{figure}

Optimizing the multi-step parameter matching objective defined in Eq.~\ref{eq3} involves backpropagation through the unrolled computational graph for $T_s$ network updates, which triggers the problem of memory efficiency, similar to meta-learning-based methods. 
To solve this problem in MTT, Cui \textit{et al.}~\cite{cui2022scaling} propose TESLA with constant memory complexity \textit{w.r.t.} update steps. 
The key step in TESLA lies in that, in the $i$-th step, $0\leq i<T_s$, it calculates the gradient for the loss function $l(\theta_{\mathcal{S}}^{(i)};\mathcal{S}^{(i)})$ against both current network parameters $\theta_{\mathcal{S}}^{(i)}$ and current synthetic samples $\mathcal{S}^{(i)}$. 
The gradient for each synthetic sample is accumulated. 
Thus, during backpropagation, there is no need to record the computational graph for $T_s$ training steps. 
Also, it is worth noting that adopting soft labels for synthetic data is crucial when condensing datasets with many classes. 

Moreover, Du \textit{et al.}~\cite{du2022minimizing} find that MTT results in accumulated trajectory error for downstream training and propose FTD. 
To be specific, in the objective of MTT shown in Eq.~\ref{eq3}, $\theta^{(0)}$ may be sampled from some training checkpoints of latter epochs on the original dataset, and the trajectory matching error on these checkpoints is minimized. 
However, it is hard to guarantee that desired parameters seen in dataset distillation optimization can be faithfully achieved when trained using the synthetic dataset, which causes the accumulated error, especially for latter epochs. 
To alleviate this problem, FTD adds flat regularization when training with original datasets, which results in flat training trajectories and makes targeting networks more robust to weight perturbation. 
Thus, it yields lower accumulated error compared with the baseline. 

\begin{figure*}[t]
\centering
\includegraphics[width=0.8\linewidth]{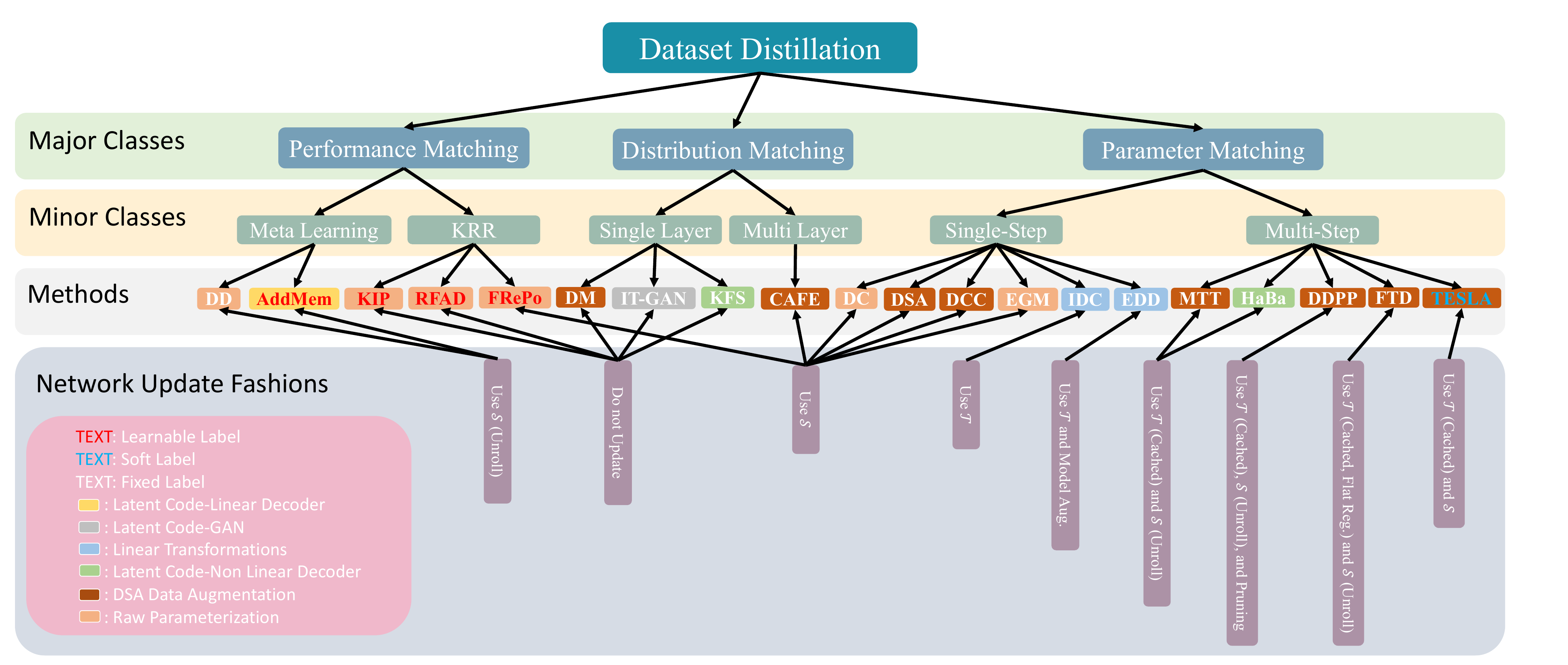}
\vspace{-0.4cm}
\caption{Taxonomy of existing DD methods. A DD method can be analyzed from 4 aspects: optimization objective, fashion of updating networks, synthetic data parameterization, and fashion of learning labels. A tabular version can be found in the supplement.}
\vspace{-0.4cm}
\label{fig:taxonomy}
\end{figure*}

\subsubsection{Distribution Matching}
The distribution matching approach aims to obtain synthetic data whose distribution can approximate that of real data. 
Instead of matching training effects, \textit{e.g.}, the performance of models trained on $\mathcal{S}$, distribution matching directly optimizes the distance between the two distributions using some metrics, \textit{e.g.}, Maximum Mean Discrepancy (MMD) leveraged in Zhao \textit{et al.}~\cite{DBLP:journals/corr/abs-2110-04181}. 
The intuition is shown in Fig.~\ref{fig:dism}. 
Since directly estimating the real data distribution can be expensive and inaccurate as images are high-dimensional data, distribution matching adopts a set of embedding functions, \textit{i.e.}, neural networks, each providing a partial interpretation of the input and their combination providing a comprehensive interpretation, to approximate MMD. 
Here, we denote the parametric function as $f_{\theta}$, and distribution matching is defined as:
\begin{equation}\small
    \text{\small $\mathcal{L}(\mathcal{S},\mathcal{T})=\mathbb{E}_{\theta\in\Theta}[\mathcal{D}(\mathcal{S},\mathcal{T};\theta)]$},\\
\end{equation}
where $\Theta$ is a specific distribution for random neural network initialization, $X_{s,c}$ and $X_{t,c}$ denote samples from the $c$-th class in synthetic and real datasets respectively, and $\mathcal{D}$ is some metric measuring the distance between two distributions. 
DM by Zhao \textit{et al.}~\cite{DBLP:journals/corr/abs-2110-04181} adopts classifier networks without the last linear layer as embedding functions. 
The center, \textit{i.e.}, mean vector, of the output embeddings for each class of synthetic and real datasets are encouraged to be close. 
Formally, $\mathcal{D}$ is defined as:
\begin{equation}\small
\begin{split}
    \text{\small $\mathcal{D}(\mathcal{S},\mathcal{T};\theta)=\sum_{c=0}^{C-1}\Vert\mu_{\theta,s,c}$}&\text{\small $-\mu_{\theta,t,c}\Vert^2$},\\
    \text{\small $\mu_{\theta,s,c}=\frac{1}{M_c}\sum_{j=1}^{M_c}f^{(i)}_{\theta}(X^{(j)}_{s,c}),\quad$}&\text{\small $\mu_{\theta,t,c}=\frac{1}{N_c}\sum_{j=1}^{N_c}f^{(i)}_{\theta}(X^{(j)}_{t,c})$},
\end{split}
\end{equation}
where $M_c$ and $N_c$ are the number of samples for the $c$-th class in synthetic and real datasets respectively, and $j$ is the sample index. 

Instead of matching features before the last linear layer, Wang \textit{et al.}\cite{wang2022cafe} propose CAFE, which forces statistics of features for synthetic and real samples extracted by each network layer except the final one to be consistent. 
The distance function is as follows:
\begin{equation}\small
      \text{\small $\mathcal{D}(\mathcal{S},\mathcal{T};\theta)=\sum_{c=0}^{C-1}\sum_{i=1}^{L-1}\Vert\mu_{\theta,s,c}^{(i)}-\mu_{\theta,t,c}^{(i)}\Vert^2$},
\end{equation}
where $L$ is the number of layers in a neural network, and $i$ is the layer index. 
Also, to explicitly learn discriminative synthetic images, CAFE adopts the discrimination loss $\mathcal{D}r$. 
Here, the center of synthetic samples for each class is used as a classifier to classify real samples. 
The discriminative loss for the $c$-th class tries to maximize the probability of this class for the classifier:
\begin{equation}\small
\begin{split}
      \text{\small $\mathcal{D}r(\mathcal{S},\mathcal{T};\theta)$}&\text{\small $=-\sum_{c=0}^{C-1}\sum_{j=1}^{N_c}\log p(c|X_{t,c}^{(j)},\mathcal{S},\theta)$},\\
      \text{\small $p(c|X_{t,c}^{(j)},\mathcal{S},\theta)$}&\text{\small $=\frac{\exp\{{\mu_{\theta,s,c}^{(L-1)}}^Tf_{\theta}^{(L-1)}(X_{t,c}^{(j)})\}}{\sum_{c'=0}^{C-1}\exp\{{\mu_{\theta,s,c'}^{(L-1)}}^Tf_{\theta}^{(L-1)}(X_{t,c}^{(j)})\}}$}.
\end{split}
\end{equation}
Meanwhile, networks are updated alternately with synthetic data by a dedicated schedule. 
Overall, the objective function can be written as:
\begin{equation}\small
\begin{split}
    \mathcal{L}(\mathcal{S},\mathcal{T})&=\mathbb{E}_{\theta^{(0)}\in\Theta}[\sum_{t=0}^T\{\mathcal{D}(\mathcal{S},\mathcal{T};\theta^{(t)})+\lambda\mathcal{D}r(\mathcal{S},\mathcal{T};\theta^{(t)})\}],\\
    \theta^{(t)} &= \theta^{(t-1)}-\eta\nabla l(\mathcal{S};\theta^{(t-1)}),
\end{split}
\end{equation}
where $\lambda$ is a hyperparameter for balance. 

\subsubsection{Connections between Objectives in DD} 
In this part, we show connections between the above three optimization objectives and give proof that they are essentially related. 
For simplicity, we assume that only the last linear layer of neural networks is considered for updating synthetic data, while previous network layers, \textit{i.e.}, the feature extractor $f_{\theta}$, are fixed. 
The parameter of the linear layer is $W \in \mathbb{R}^{F \times C}$, where $F$ is the dimension for of feature from $f_{\theta}$, and $C$ is the number of classes. 
Least square function is adopted for model optimization, as the case in FRePo~\cite{zhou2022dataset}. 

\textbf{Performance Matching v.s. Optimal Parameter Matching.}
For performance matching, the goal is that optimize the performance on $\mathcal{T}$ for models trained on $\mathcal{S}$, as shown in Eq.~\ref{eq:per}. 
Here, for simplicity, we ignore the regularization term and assume $\lambda = 0$, given that $\lambda$ is generally a small constant for numerical stability. 
Then, the performance matching objective in Eq.~\ref{eq:per} for some given $\theta$ can be written as:
\begin{equation}\small
    \mathcal{L}_{perfM}=\Vert Y_t - f_\theta(X_t)f_\theta(X_s)^T(f_\theta(X_s)f_\theta(X_s)^T)^{-1}Y_s\Vert^2.
\end{equation}
Under this circumstance, we have the following proposition:
\begin{prop}
Performance matching objective for kernel ridge regression models is equivalent to optimal parameter matching objective, or infinity-step parameter matching. 
\end{prop}
\begin{proof}
Given that KRR is a convex optimization problem and the optimal analytical solution can be achieved by a sufficient number of training steps, the optimal parameter matching, in this case, is essentially infinity-step parameter matching, which can be written as:
\begin{equation}\small
   \begin{split}
       \mathcal{L}_{paraM} &= \Vert W_{\mathcal{S},\theta}^* - W_{\mathcal{T},\theta}^*\Vert^2,\\
       W_{\mathcal{S},\theta}^*&= \mathop{\arg\min}\limits_{W_{\mathcal{S},\theta}}\{\Vert Y_s-f_\theta(X_s)W_{\mathcal{S},\theta}\Vert^2\}\\
    &=f_\theta(X_s)^T(f_\theta(X_s)f_\theta(X_s)^T)^{-1}Y_s,\\
    W_{\mathcal{T},\theta}^*&= \mathop{\arg\min}\limits_{W_{\mathcal{T},\theta}}\{\Vert Y_t-f_\theta(X_t)W_{\mathcal{T},\theta}\Vert^2\}\\
    &=(f_\theta(X_t)^Tf_\theta(X_t))^{-1}f_\theta(X_t)^TY_t.\\ 
   \end{split} 
\end{equation}
Here, we denote $(f_\theta(X_t)^Tf_\theta(X_t))^{-1}f_\theta(X_t)^T$ as $\mathcal{M}$. 
Then we get: 
\begin{equation}\small
\label{eq:perpara}
    \Vert\mathcal{M}\Vert^2\mathcal{L}_{perfM} = \mathcal{L}_{paraM}.
\end{equation}
As $\mathcal{M}$ is a constant matrix and not related to the optimization problem \textit{w.r.t} $\mathcal{S}$, concluded from Eq.~\ref{eq:perpara}, performance matching is equivalent to optimal parameter matching under this circumstance. 
\end{proof}

\textbf{Single-Step Parameter Matching v.s. First-Moment Distribution Matching.} 
In this part, we reveal connections between single-step parameter matching, \textit{i.e.}, gradient matching, and distribution matching based on the first moment, \textit{e.g.}, \cite{DBLP:journals/corr/abs-2110-04181}. 
Some insights are adapted from Zhao \textit{et al.}~\cite{DBLP:journals/corr/abs-2110-04181}. 
We have the following proposition:
\begin{prop}
    First-order distribution matching objective is approximately equal to gradient matching of each class for kernel ridge regression models following a random feature extractor. 
\end{prop}
\begin{proof}
Denote the parameter of the final linear layer as $W = [{w}_0, {w}_1, \dots, {w}_{C-1}]$, where ${w}_c \in \mathbb{R}^F$ is the weight vector connected to the $c$-th class. 
And for each training sample $({x}_j)$ with class label $c$, we get the loss function as follows:
\begin{equation}\small
\begin{split}
    l &= \Vert y_i - f_\theta(x_j)W\Vert^2=\sum_{c'=0}^{C-1}((\mathbbm{1}_{c'=c}-f_\theta(x_j)w_{c'})^2,\\
\end{split}
\end{equation}
where $y_i$ is a one-hot vector with the $c$-th entry $1$ and others $0$, and $\mathbbm{1}$ is the indicator function. 
Then, the partial derivative of the classification loss on $j$-th sample \textit{w.r.t} the $c'$-th neuron is:
\begin{equation}\small
\label{eqg}
    {g}_{c',j} = (p_{c',j}-\mathbbm{1}_{c'=c})\cdot f_\theta(x_j)^T,
\end{equation}
where $p_{c',j}=f_\theta(x_j)w_{c'}$. 
For seminal works on gradient matching~\cite{zhao2021dataset} and distribution matching~\cite{DBLP:journals/corr/abs-2110-04181}, objective functions are calculated for each class separately. 
For the $c$-th class, the gradient of the $c'$-th weight vector $w_{c'}$ is:
\begin{equation}\small
\begin{split}
    g_{c'}&=\frac{1}{N_c}\sum_{j=1}^{N_c}{g}_{c',j}=\frac{1}{N_c}\sum_{j=1}^{N_c}(p_{c',j}-\mathbbm{1}_{c'=c})\cdot f_\theta(x_j)^T\\
    &\approx (q-\mathbbm{1}_{c'=c})\frac{1}{N_c}\sum_{j=1}^{N_c}f_\theta(x_j)^T,
\end{split}
\label{eqg2}
\end{equation}
where the approximation is valid since the network parameterized by $\theta$ is randomly initialized, and the prediction for each class is near uniform, denoted as $q$. 
The term $\frac{1}{N_c}\sum_{j=1}^{N_c}f_\theta(x_j)^T$ is actually the first moment of the output distribution for samples of the $c$-th class, \textit{i.e.}, $\mu_{\theta,s,c}$ for $\mathcal{S}$ and $\mu_{\theta,t,c}$ for $\mathcal{T}$. 
Thus, first-moment distribution matching is approximately equivalent to gradient matching for each class in this case. 
\end{proof}

\textbf{Single-Step Parameter Matching v.s. Second-Order Distribution Matching.} 
The relationship between gradient matching and first-order distribution matching discussed above relies on an important assumption that the outputs of networks are uniformly distributed, which requires fully random neural networks without any learning bias. 
Here we show in the following proposition that without such a strong assumption, we can still bridge gradient matching and distribution matching. 
\begin{prop}
Second-order distribution matching objective optimizes an upper bound of gradient matching objective for kernel ridge regression models. 
\end{prop}
\begin{proof}
The gradient matching objective in this case satisfies:
\begin{equation}\small
\small
\begin{split}
    \mathcal{L}=&\Vert\frac{1}{|\mathcal{S}|}\nabla l(f_\theta(X_s), W)-\frac{1}{|\mathcal{T}|}\nabla l(f_\theta(X_t), W)\Vert^2\\
    =&\Vert\frac{1}{|\mathcal{S}|}f_\theta(X_s)^T(f_\theta(X_s)W-Y_s)\\&-\frac{1}{|\mathcal{T}|} f_\theta(X_t)^T(f_\theta(X_t)W-Y_t)\Vert^2\\
    =&\Vert(\frac{1}{|\mathcal{S}|}f_\theta(X_s)^Tf_\theta(X_s)-\frac{1}{|\mathcal{T}|}f_\theta(X_t)^Tf_\theta(X_t))W\\
     &-(\frac{1}{|\mathcal{S}|}f_\theta(X_s)^TY_s-\frac{1}{|\mathcal{T}|}f_\theta(X_t)^TY_t)\Vert^2 \\
     \leq&\Vert\frac{1}{|\mathcal{S}|}f_\theta(X_s)^Tf_\theta(X_s)-\frac{1}{|\mathcal{T}|}f_\theta(X_t)^Tf_\theta(X_t)\Vert^2\Vert W\Vert^2 \\
     &+\Vert\frac{1}{|\mathcal{S}|}f_\theta(X_s)^TY_s-\frac{1}{|\mathcal{T}|}f_\theta(X_t)^TY_t\Vert^2.
\end{split}
\label{eqg3}
\end{equation}
The $r.h.s$ of the inequality of Eq.~\ref{eqg3} measures the difference between first and second-order statistics of synthetic and real data, \textit{i.e.}, mean, and correlation. 
In this way, distribution matching essentially optimizes an upper bound of gradient matching. 
\end{proof}

\subsection{Synthetic Data Parameterization}

\begin{figure}[t]
\centering
\includegraphics[width=\linewidth]{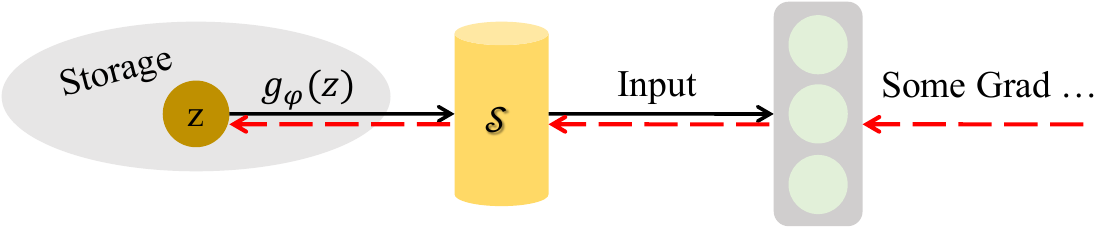} 
\vspace{-0.6cm}
\caption{Synthetic data parameterization. It is a sub-area in DD and orthogonal to the objectives for dataset distillation, since the process for generating  $\mathcal{S}$ is differentiable and gradients for updating $\mathcal{S}$ can be further backpropagated to synthetic parameters.}
\vspace{-0.6cm}
\label{fig:parameterization}
\end{figure}

One of the essential goals of dataset distillation is to synthesize informative datasets to improve training efficiency, given a limited storage budget. 
In other words, as for the same limited storage, more information on the original dataset is expected to be preserved so that the model trained on condensed datasets can achieve comparable and satisfactory performance. 
In the image classification task, the typical method of dataset distillation is to distill the information of the dataset into a few synthetic images with the same resolution and number of channels as real images. 
However, due to the limited storage, information carried by a small number of images is limited. 
Moreover, with the same format as real data, it is unclear whether synthetic data contain useless or redundant information. 
Focusing on these concerns and orthogonal to optimization objectives in DD, a series of works propose different ways of synthetic data parameterization. 
In a general form, for a synthetic dataset, some codes $z\in \mathcal{Z} \subset \mathbb{R}^{D'}$, $\mathcal{Z}=\{(z_j, y_j)\}|^{|\mathcal{Z}|}_{j=1}$ in a format other than the raw shape are used for storage. 
And there is some function $g_{\phi}: \mathbb{R}^{D'} \rightarrow \mathbb{R}^D$ parameterized by $\phi$ that maps a code with $d'$ dimensions to the format of raw images for downstream training. 
In this case, the synthetic dataset is denoted as:
\begin{equation}\small
    \mathcal{S}=(X_s,Y_s),X_s=\{g_{\phi}(z_j)\}^{|\mathcal{S}|}_{j=1}.\label{eq:param}
\end{equation}
As shown in Fig.~\ref{fig:parameterization}, when training, $\phi$ and $\mathcal{Z}$ can be updated in an end-to-end fashion, by further backpropagating the gradient for $\mathcal{S}$ to them, since the process for data generation is differentiable. 

\subsubsection{Differentiable Siamese Augmentation}\label{dsa}
Differentiable siamese augmentation~(DSA) is a set of augmentation policies designed to improve the data efficiency, including crop~\cite{krizhevsky2017imagenet}, cutout~\cite{devries2017improved}, flip, scale, rotate, and color jitters~\cite{krizhevsky2017imagenet} operations. 
It is first applied to the dataset distillation task by Zhao \textit{et al.}~\cite{zhao2021datasetb} to increase the data efficiency and thus generalizability. 
Here, $g_\phi(\cdot)$ is a family of image transformations parameterized with $\phi\sim\Phi$, where $\phi$ is the parameter for data augmentation, and a code $z$ still holds the same format as a raw image. 
It is differentiable so that the synthetic data can be optimized via gradient descent. 
Moreover, the parameter for data augmentation is the same for synthetic and real samples within each iteration to avoid the averaging effect. 
This technique has been applied in many following DD methods, \textit{e.g.}, DM~\cite{DBLP:journals/corr/abs-2110-04181}, MTT~\cite{cazenavette2022dataset}, TESLA~\cite{cui2022scaling}, \textit{etc}. 

\subsubsection{Upsampling}
According to experimental observations, Kim \textit{et al.}~\cite{kim2022dataset} find that the size of the synthetic dataset rather than resolution dominates its performance. 
Motivated by this, they propose a synthetic data parameterization strategy termed IDC, which uses multi-formation to increase the number of synthetic data under the same storage budget by reducing the data resolution. 
Intuitively, the format of $z$ is a down-sampled format of a raw image, \textit{e.g.}, $1/2$ down-sampling such that the number of synthetic images can become quadrupled, and $g_{\phi}$ here is a parameter-free up-sampling function, whose computational overhead is negligible. 

\subsubsection{Generators and Latent Vectors}
A generator is capable of producing data from some latent vectors, which provide a more compact representation of raw data by learning hidden patterns of data formation. 
Parameterizations based on generators and latent vectors can be regarded as a decomposition of shared knowledge and distinct knowledge across various samples, where generators extract common patterns while different latent vectors store different magnitudes for these patterns. 
The dimension of latent vectors is typically lower than that of original formats. 
Thus, given the same storage budget, more synthetic data can be obtained by storing generators and latent vectors instead of raw samples, and more information on original datasets can be preserved, which indicates a more data-efficient parameterization.
In this case, $z$ in Eq.~\ref{eq:param} denotes latent codes and $g_{\phi}$ denotes the generator parameterized by $\phi$. 

\textbf{GAN.} 
Zhao \textit{et al.}~\cite{zhao2022synthesizing} propose to generate informative synthetic samples via GAN. 
In this work, $g_\phi$ is a pretrained and fixed GAN generator. 
Motivated by the observation that latent vectors obtained by GAN inversion are more informative than those sampled randomly for downstream training, the method learns a set of latent vectors initialized with GAN inversion of raw data. 
In specific, it first initializes the whole latent set $\mathcal{Z}$ by GAN inversion~\cite{abdal2019image2stylegan}, such that the generated result $g_{\phi}(z_j)$ with each latent vector $z_j$ is as close as possible to the corresponding real image in the original training dataset $\mathcal{T}$:
\begin{equation}\small
    \mathop{\arg\min}\limits_{z}\Vert h(g_\phi(z), \omega) - h(x, \omega)\Vert^2+\lambda\Vert g_\phi(z) - x\Vert^2,
\end{equation}
where $h(\cdot, \omega)$ is a pre-trained and fixed feature extractor with parameter $\omega$, and $\lambda$ is the weight to balance the two terms. 
Then in the training phase, the objective function is based on the combination of DM and a regularization term:
\begin{equation}\small
\begin{split}
    \mathcal{L}&=(1-\lambda')\mathcal{L}_{DM} + \lambda' R, \\
    \mathcal{L}_{DM}&=\mathbb{E}_{\theta\in\Theta}[\Vert\frac{1}{N}\sum_{j=1}^{N}f_{\theta}(g_{\phi}(z_j))-\frac{1}{N}\sum_{j=1}^{N}f_{\theta}(x_j)\Vert^2],\\
    R&=\mathbb{E}_{\theta\in\Theta}[\sum_{j=1}^{N}\Vert f_\theta(g_\phi(z_j)) - f_\theta(x_j)\Vert^2],
\end{split}
\end{equation}
where the $x_i$ and $z_i$ are paired, and $\lambda'$ is another hyperparameter for balance. 
The regularization item can better preserve the training information compared to the feature alignment in GAN inversion. 

\textbf{Addressing Matrices.}
Deng \textit{et al.}~\cite{deng2022remember} propose compressing the critical information of original datasets into compact addressable memories. 
Instead of learning each synthetic sample separately, it learns a common memory representation for all classes and samples, which can be accessed through learnable addressing matrices to construct the synthetic dataset. 
Thus, the size of the synthetic dataset does not necessarily correlate with the number of classes linearly, and the high compression rate achieves better data efficiency. 
In this case, the shared common memory, or basis, serves as the parameter $\phi$ for the mapping function $g$, and the address, or coefficient, serves as the latent code $z$ for each sample. 
Specifically, assume that there are $C$ classes, and we want to retrieve $R$ samples for each class. 
Then we have $R$ addressing matrices $A_j\in\mathbb{R}^{C\times K}$ for $1\leq j\leq R$. 
$K$ is the number of bases vectors, and the memory $\phi\in\mathbb{R}^{K\times D}$, where $D$ is the dimension for each sample. 
The $j$-th synthetic sample for the $c$-th class $x_{c,j}$ is constructed by the following linear transformation:
\begin{equation}\small
    z_{c,j}=y_cA_j,\quad x_{c,j}=z_{c,j}\phi,
\end{equation}
where $y_c$ is the corresponding one-hot vector. 
For optimization, the method proposes an improved meta learning based performance matching framework with momentum used for inner loops. 

\textbf{Decoders.}
Beyond the linear transformation-based method by Deng \textit{et al.}~\cite{deng2022remember}, HaBa~\cite{liu2022dataset} and KFS~\cite{lee2022dataset} adopt non-linear transformation for the mapping function $g$. 
HaBa uses the multi-step parameter matching method MTT~\cite{cazenavette2022dataset} for optimization while KFS is based on distribution matching~\cite{DBLP:journals/corr/abs-2110-04181}. 
Nevertheless, they use similar techniques for synthetic data parameterization. 
Specifically, there are $|\mathcal{Z}|$ latent codes (bases) and $|\Phi|$ neural decoders (hallucinators). 
Latent codes may include distinct knowledge for different samples, while decoders contain shared knowledge across all samples. 
A synthetic sample can be generated by sending the $j$-th latent code $z_j$ to the $k$-th decoder parameterized by $\phi_k$:
\begin{equation}\small
    x_{j,k}=g_{\phi_k}(z_j),
\end{equation}
where $z_j\in\mathcal{Z}$, $1\leq j\leq |\mathcal{Z}|$ and $\phi_k\in\Phi$, $1\leq k\leq |\Phi|$. 
In other words, different latent codes and decoders are combined interchangeably and arbitrarily, such that the size of the synthetic dataset can reach up to $|\mathcal{Z}|\times|\Phi|$, which increases the data efficiency exponentially. 

\subsection{Label Distillation Methods.} \label{label distill}
In dataset distillation for image classification tasks, most methods fix the class label $Y_s$ as a one-hot format and only learn synthetic images $X_s$. 
Some works~\cite{sucholutsky2021soft,nguyen2020dataset,deng2022remember,zhou2022dataset} find that making labels learnable can improve performance. 
Bohdal \textit{et al.}~\cite{bohdal2020flexible} reveal that even only learning labels without learning images can achieve satisfactory performance. 
Moreover, Cui \textit{et al.}~\cite{cui2022scaling} takes soft labels predicted by a teacher model trained on real datasets, which is demonstrated to be significantly effective when condensing datasets with a large number of classes. 
Generally, as plug-and-play schemes, how to deal with labels in DD is orthogonal to optimization objectives and is compatible with all existing ones. 

\section{Applications}\label{sec:application}
Benefiting from the nature of the high compression rate, the research of dataset distillation has led to many successful innovative applications in various fields like continual learning and federated learning, which will be introduced in this section. 

\subsection{Continual Learning}
Continual learning~\cite{kirkpatrick2017overcoming} aims to remedy the catastrophic forgetting problem caused by the inability of neural networks to learn a stream of tasks in sequence. 
A widely used strategy is rehearsal, maintaining representative samples that can be re-used to retain knowledge about previous tasks~\cite{buzzega2021rethinking,liu2021adaptive,rebuffi2017icarl,prabhu2020gdumb}. 
In this case, how to preserve as much knowledge of previous data in a buffer space with limited memory becomes an essential point. 
Benefiting from the high compression rate, the distilled data can capture the essence of the whole dataset, making it an inspiring application in continual learning. 
Several works~\cite{zhao2021dataset,DBLP:journals/corr/abs-2110-04181,liu2020mnemonics,zhou2022dataset,rosasco2022distilled} have directly applied dataset distillation techniques to continual learning scenarios. 
Instead of selecting the representative samples of historical data, they train synthetic datasets for historical data to keep the knowledge of the past task. 

Besides, some synthetic dataset parameterization methods are also successfully applied to improve memory efficiency. 
Kim \textit{et al.}~\cite{kim2022dataset} use data partition and upsampling to increase the number of training samples under the same storage budget. 
Deng \textit{et al.}~\cite{deng2022remember} propose using addressable memories and addressing matrices to construct synthetic data of historical tasks, where different tasks can share the same memory base so that the memory cost will not increase linearly or positively correlate with the number of past tasks. 
Sangermano \textit{et al.}~\cite{sangermano2022sample} propose obtaining the synthetic dataset by linear weighted combining historical images, and the optimization target is linear coefficients rather than pixels of synthetic data. 
Wiewel \textit{et al.}~\cite{wiewel2021condensed} propose a method that learns a weighted combination of shared components for samples in a specific class rather than directly learning the synthetic dataset. 
Masarczyk \textit{et al.}~\cite{masarczyk2020reducing} 
use the generative model to create the synthetic dataset and form a sequence of tasks to train the model, and the parameters of the generative model are fine-tuned by the evaluation loss of the learner on real data. 

\subsection{Federated Learning}
Federated learning (FL)~\cite{konevcny2016federated,mcmahan2017communication,yang2019federated} develops a privacy-preserving distributed model training scheme such that multiple clients collaboratively learn a model without sharing their private data. 
A standard way of federated learning to keep data private is to transmit model updates instead of private user data. 
However, this may cause an increased communication cost, for the size of model updates may be very large and make it burdensome for clients when uploading frequently. 

To remedy this problem and improve communication efficiency, some research~\cite{goetz2020federated,zhou2020distilled,xiong2022feddm,song2022federated,liu2022meta,hu2022fedsynth} on federated learning transmits the locally generated synthetic datasets instead of model updates from clients to the server. 
The motivation is quite straightforward: the model parameters are usually much larger than a small number of data points which can reduce the costs significantly in upload transmission from clients back to the server, and the synthetic dataset can preserve the essence of the whole original dataset. 
It is worth noting that label distillation~\cite{bohdal2020flexible} is usually considered to prevent synthetic data from representing an actual label for privacy issues~\cite{goetz2020federated}. 
To further strengthen privacy protection, Zhou \textit{et al.}~\cite{zhou2020distilled} combine dataset distillation and distributed one-shot learning, such that for every local update step, each synthetic data successively updates the network for one gradient descent step. 
Thus, synthetic data is closely bonded with one specific network weight, and the eavesdropper cannot reproduce the result with only leaked synthetic data. 
Xiong \textit{et al.}~\cite{xiong2022feddm} adopt distribution matching~\cite{DBLP:journals/corr/abs-2110-04181} to generate synthetic data and update synthetic data with the Gaussian mechanism to protect the privacy of local data. 
Song \textit{et al.}~\cite{song2022federated} apply dataset distillation in one-shot federated learning and propose a novel evaluation metric $\gamma$-accuracy gain to tune the importance of accuracy and analyze communication efficiency. 
Liu \textit{et al.}~\cite{liu2022meta} develop two mechanisms for local updates: dynamic weight assignment, when training the synthetic dataset, assigning dynamic weights to each sample based on its training loss; meta-knowledge sharing, sharing local synthetic dataset among clients to mitigate heterogeneous data distributions among clients. 

In addition to reducing the communication cost of each round, dataset distillation can also be applied to reduce the communication epochs to reduce the total communication consumption. 
Pi \textit{et al.}~\cite{pi2022dynafed} propose to use fewer rounds of standard federated learning to generate the synthetic dataset on the server by applying multi-step parameter matching on global model trajectories and then using the generated synthetic dataset to complete subsequent training.

\subsection{Neural Architecture Search}
Neural architecture search (NAS)~\cite{elsken2019neural,wistuba2019survey} aims to discover the optimal structure of a neural network from some search space. 
It typically requires expensive training of numerous candidate neural network architectures on full datasets. 

In this case, datasets by DD, benefiting from its small size, can be served as proxy sets to accelerate model evaluation in neural architecture search, which is proved feasible in some works~\cite{zhao2021dataset,zhao2021datasetb,DBLP:journals/corr/abs-2110-04181}. 
Besides, Such \textit{et al.}~\cite{such2020generative} propose a method named generative teaching networks combining a generative model and a learner to create the synthetic dataset, which is learner-agnostic, \textit{i.e.}, generalizes to different latent learner architectures and initializations. 

\subsection{Privacy, Security and Robustness}\label{sec:privacy}
Machine learning suffers from a wide range of privacy attacks~\cite{lyu2020threats}, \textit{e.g.}, model inversion attack~\cite{fredrikson2015model}, membership inference attack~\cite{shokri2017membership}, property inference attack~\cite{melis2019exploiting}. 
Dataset distillation provides a perspective to start from the data alone, protect the privacy and improve model robustness.

There are some straightforward applications of dataset distillation to protect the privacy of the dataset, for synthetic data may look unreal to recognize the actual label, and it is hard to reproduce the result with synthetic data without knowing the architecture and initialization of the target model. 
For example, remote training~\cite{sucholutsky2020secdd} transmits synthetic datasets with distilled labels instead of the original dataset to protect data privacy.

More for data privacy protection, Dong \textit{et al.}~\cite{dong2022privacy} point out that dataset distillation can offer privacy protection to prevent unintentional data leakage. 
They emerge dataset distillation techniques into the privacy community and give the theoretical analysis of the connection between dataset distillation and differential privacy. 
Also, they empirically validate that the synthetic dataset is irreversible to original data in terms of similarity metrics of $L_2$ and LPIPS~\cite{kettunen2019lpips}. 
Chen \textit{et al.}~\cite{chen2022private} apply dataset distillation to generate high-dimensional data with differential privacy guarantees for private data sharing with lower memory and computation costs. 

As for the robustness of model improvement, Tsilivis \textit{et al.}~\cite{tsilivis2022can} proposes that dataset distillation can provide a new perspective for solving robust optimization. 
They combine the adversarial training and KIP~\cite{nguyen2020dataset,nguyen2021dataset} to optimize the training data instead of model parameters with high efficiency. It is beneficial that the optimized data can be deployed with other models and give favorable transferability. 
Also, it enjoys satisfactory robustness against PGD attacks~\cite{madry2017towards}. 
Huang \textit{et al.}~\cite{https://doi.org/10.48550/arxiv.2211.10752} study the robust dataset learning problem such that the network trained on the dataset is adversarially robust. 
They formulate robust dataset learning as a min-max, tri-level optimization problem where the robust error of adversarial data on the robust data parameterized model is minimized.

Furthermore, traditional backdoor attacks, which inject triggers into the original data and use the malicious data to train the model, cannot work on the distilled synthetic dataset as it is too small for injection, and inspection can quickly mitigate such attacks. 
However, Liu \textit{et al.}~\cite{liu2023backdoor} propose a new backdoor attack method, DOORPING, which attacks during the dataset distillation process rather than afterward model training. 
Specifically, DOORPING continuously optimizes the trigger in every epoch before updating the synthetic dataset to ensure that the trigger is preserved in the synthetic dataset. 
Experiment results show that nine defense mechanisms for the backdoor attacks at three levels, \textit{i.e.}, model-level~\cite{li2021neural,liu2019abs,wang2019neural}, input-level~\cite{cho2020dapas,gao2019strip,kwon2021defending} and dataset-level~\cite{hayase2021spectre,tang2019demon,tran2018spectral}, are unable to mitigate the attacks effectively.

\subsection{Graph Neural Network}
Graph neural networks (GNNs)~\cite{battaglia2018relational,kipf2016semi,velivckovic2017graph,wu2020comprehensive} are developed to analyze graph-structured data, representing many real-world data such as social networks~\cite{fan2019graph} and chemical molecules~\cite{ying2018hierarchical}. 
Despite their effectiveness, similar to traditional deep neural networks, GNNs suffer from data-hungry problems: large-scale datasets are required to learn powerful representations. 
Motivated by dataset distillation, graph-structured data can also be distilled into a synthetic and simplified one to improve the training efficiency while preserving the performance. 

Jin \textit{et al.}~\cite{jin2021graph} apply gradient matching~\cite{zhao2021dataset} in graph distillation, distilling both graph structure and node attributes. 
It is worth noting that the graph structure and node features may have connections, such that node features can represent the graph structure matrix~\cite{pfeiffer2014attributed}. 
In this case, the optimization target is only node features, which avoids the quadratic increase of computation complexity with the number of synthetic graph nodes. 
Based on that, Jin \textit{et al.}~\cite{jin2022condensing} propose a more efficient graph condensation method via one-step gradient matching without training networks. 
To handle the discrete data, they formulate the graph structure as a probabilistic model that can be learned in a differentiable manner. 
Liu \textit{et al.}~\cite{liu2022graph} adopt distribution matching~\cite{DBLP:journals/corr/abs-2110-04181} in graph condensation, synthesizing a small graph that shares a similar distribution of receptive fields with the original graph. 

\subsection{Recommender System}
Recommender systems~\cite{konstan2012recommender,isinkaye2015recommendation} aim to give users personalized recommendations, content, or services through a large volume of dynamically generated information, such as item features, user past behaviors, and similar decisions made by other users. 
The recommender system research also faces a series of challenges on heavy computational overheads caused by training models on massive datasets, which can involve billions of user-item interaction logs. 
At the same time, the security of user data should also be considered~\cite{narayanan2008robust}. 
In this case, Sachdeva \textit{et al.}~\cite{sachdeva2022infinite} develop a data distillation framework to distill the collaborative filtering data into small, high-fidelity data summaries. 
They take inspiration from KIP~\cite{nguyen2020dataset,nguyen2021dataset} to build an efficient framework. 
To deal with the discrete nature of the recommendation problem, instead of directly optimizing the interaction matrix, this method learns a continuous prior for each user-item pair, and the interaction matrix is sampled from the learned distribution. 

\subsection{Text Classification}
Text classification~\cite{minaee2021deep} is one of the classical problems in natural language processing, which aims to assign labels or tags to textual units. 
The latest developed language models~\cite{brown2020language}, which require massive datasets, are burdensome to train, fine-tune and use.
Motivated by dataset distillation, it is possible to generate a much smaller dataset to cover the knowledge in the original dataset and provide a more efficient dataset to train models. 
Nevertheless, the discrete nature of text data makes dataset distillation a challenging problem. 
In this case, Li \textit{et al.}~\cite{li2021data} propose to generate words embeddings instead of actual text words to form synthetic datasets and adopt performance matching~\cite{wang2018dataset} as the optimization objective.
Moreover, label distillation can also be applied to the text classification task. 
Sucholutsky \textit{et al.}~\cite{sucholutsky2021soft} propose to embed text into a continuous space, and then train the embedded text data with the soft-label image distillation method.

\subsection{Knowledge Distillation}
Most of the existing KD studies~\cite{hinton2015distilling,romero2014fitnets,zagoruyko2016paying,tian2019contrastive,chen2021distilling} transfer the knowledge hints of the whole sample space during the training process while neglecting the variation of student's capacity over the training progress. 
Such redundant knowledge causes two issues: 1. more computational costs, \textit{e.g.}, memory footprint and GPU time; 2. poor performance, distracting the attention of the student model from the proper knowledge and weakening the learning efficiency. 
In this case, Li \textit{et al.}~\cite{li2022knowledge} propose a novel KD paradigm of knowledge condensation, which performs knowledge condensation and model distillation alternately. 
A sample is condensed according to its value determined by the feedback from the student model iteratively. 


\subsection{Medical} 
Sharing of medical datasets is crucial to establish the cross-hospital flow of medical information and improve the quality of medical services~\cite{kumar2021integration}, \textit{e.g.}, constructing high-accuracy computer-aided diagnosis systems~\cite{weitzman2010sharing}. 
However, potential issues on privacy protection~\cite{kaissis2020secure}, transmission, and storage costs are unneglectable. 
To solve this problem, Li \textit{et al.}~\cite{li2020soft,li2022dataset,li2022compressed} leverage dataset distillation to extract the essence of a dataset and construct a much smaller anonymous synthetic dataset for data sharing. 
They handle large-size and high-resolution medical datasets by dividing them into several patches and labeling them into different categories manually. 
Also, they successfully apply performance matching~\cite{wang2018dataset}, and multi-step parameter matching~\cite{cazenavette2022dataset} objectives, along with the label distillation~\cite{bohdal2020flexible} to generate synthetic medical datasets. 

\subsection{Fashion, Art and Design}
The images produced by dataset distillation have a certain degree of visual appreciation. 
It retains the characteristics and textures of the given category and produces artistic fragments. 
Cazenavette \textit{et al.}~\cite{cazenavette2022wearable} propose a method that generates tileable distilled texture, which is aesthetically pleasing to be applied to practical tasks, \textit{e.g.}, clothing. 
They update the canvas by applying random crops on a padded distilled canvas and performing dataset distillation on those crops. 
For fashion compatibility learning, Chen \textit{et al.}~\cite{chen2022learning} propose using designer-generated data to guide outfit compatibility modeling. 
They extract disentangled features from a set of fashion outfits, generate the fashion graph, and leverage the designer-generated data through a dataset distillation scheme, which benefits the fashion compatibility prediction.

\begin{table*}[t]
\tiny
\renewcommand{\arraystretch}{0.8}
\centering
\caption{Comparison for different dataset distillation methods on the same architecture with training. $\dag$ adopt momentum~\cite{deng2022remember}. Please refer to the supplementary for the performance of all the existing DD methods.}
\vspace{-0.2cm}
\resizebox{\linewidth}{!}{
\begin{tabular}{c l c c c c c c c}
\toprule
\multirow{3}{*}{\bf Dataset} & \multirow{3}{*}{\bf Img/Cls} & \multicolumn{6}{c}{\bf Method} \\
\cmidrule{3-8}
& &DD$^\dag$~\cite{wang2018dataset,deng2022remember} & DC~\cite{zhao2021dataset} & DSA~\cite{zhao2021datasetb} & DM~\cite{DBLP:journals/corr/abs-2110-04181} & MTT~\cite{cazenavette2022dataset} & FRePo~\cite{zhou2022dataset}\\
\midrule
\multirow{3}{*}{\bf MNIST}
&1 & \textbf{95.2 $\pm$ 0.3} & 91.7 $\pm$ 0.5 & 88.7 $\pm$ 0.6 & 89.9 $\pm$ 0.8 & 91.4 $\pm$ 0.9 & 93.8 $\pm$ 0.6\\
&10 & 98.0 $\pm$ 0.1 & 97.4 $\pm$ 0.2 & 97.9 $\pm$ 0.1 & 97.6 $\pm$ 0.1 & 97.3 $\pm$ 0.1 & \textbf{98.4 $\pm$ 0.1}\\
&50 & 98.8 $\pm$ 0.1 & 98.8 $\pm$ 0.2 & \bf 99.2 $\pm$ 0.1 & 98.6 $\pm$ 0.1 & 98.5 $\pm$ 0.1 & \textbf{99.2 $\pm$ 0.1}\\
\midrule
\multirow{3}{*}{\bf Fashion-MNIST}
&1 & \textbf{83.4 $\pm$ 0.3} & 70.5 $\pm$ 0.6 & 70.6 $\pm$ 0.6 & 71.5 $\pm$ 0.5 & 75.1 $\pm$ 0.9 & 75.6 $\pm$ 0.5\\
&10 & \textbf{87.6 $\pm$ 0.4} & 82.3 $\pm$ 0.4 & 84.8 $\pm$ 0.3 & 83.6 $\pm$ 0.2 & 87.2 $\pm$ 0.3 & 86.2 $\pm$ 0.3\\
&50 & 87.7 $\pm$ 0.3 & 83.6 $\pm$ 0.4 & 88.8 $\pm$ 0.2 & 88.2 $\pm$ 0.1 & 88.3 $\pm$ 0.1 & \textbf{89.6 $\pm$ 0.1}\\
\midrule
\multirow{3}{*}{\bf CIFAR-10}
&1 & 46.6 $\pm$ 0.6 & 28.3 $\pm$ 0.5 & 28.8 $\pm$ 0.7 & 26.5 $\pm$ 0.4 & 46.3 $\pm$ 0.8 & \textbf{46.8 $\pm$ 0.7}\\
&10 & 60.2 $\pm$ 0.4 & 44.9 $\pm$ 0.5 & 53.2 $\pm$ 0.8 & 48.9 $\pm$ 0.6 & 65.3 $\pm$ 0.7 & \textbf{65.5 $\pm$ 0.6}\\
&50 & 65.3 $\pm$ 0.4 & 53.9 $\pm$ 0.5 & 60.6 $\pm$ 0.5 & 63.0 $\pm$ 0.4 & 71.6 $\pm$ 0.2 & \textbf{71.7 $\pm$ 0.2}\\
\midrule
\multirow{3}{*}{\bf CIFAR-100}
&1 & 19.6 $\pm$ 0.4 & 12.6 $\pm$ 0.4 & 13.9 $\pm$ 0.3 & 11.4 $\pm$ 0.3 & 24.3 $\pm$ 0.3 & \textbf{27.2 $\pm$ 0.4}\\
&10 & 32.7 $\pm$ 0.4 & 25.4 $\pm$ 0.3 & 32.3 $\pm$ 0.3 & 29.7 $\pm$ 0.3 & 40.1 $\pm$ 0.4 & \textbf{41.3 $\pm$ 0.2} \\
&50 & 35.0 $\pm$ 0.3 & 29.7 $\pm$ 0.3 & 42.8 $\pm$ 0.4 & 43.6 $\pm$ 0.4 & \textbf{47.7 $\pm$ 0.2} & 44.3 $\pm$ 0.2\\
\midrule
\multirow{3}{*}{\bf Tiny-ImageNet}
&1 & - & 5.3 $\pm$ 0.2 & 6.6 $\pm$ 0.2 & 3.9 $\pm$ 0.2 & 8.8 $\pm$ 0.3 & \textbf{15.4 $\pm$ 0.3}\\
&10 & - & 11.1 $\pm$ 0.3 & 16.3 $\pm$ 0.2 & 13.5 $\pm$ 0.3 & 23.2 $\pm$ 0.2 & \textbf{24.9 $\pm$ 0.2}\\
&50 & - & 11.2 $\pm$ 0.3 & 25.3 $\pm$ 0.2 & 24.1 $\pm$ 0.3 & \bf 28.2 $\pm$ 0.5 & -\\
\bottomrule
\end{tabular}
}
\label{tab:performance}
\end{table*}

\begin{table*}[t]
\renewcommand{\arraystretch}{0.9}
\centering
\tiny
\caption{Cross-architecture transfer performance on CIFAR-10 with 10 Img/Cls. ConvNet is the default evaluation model used for each method. NN, IN and BN stand for no normalization, Instance Normalization, and Batch Normalization respectively. $\dag$ adopt momentum~\cite{deng2022remember}.}
\vspace{-0.2cm}
\resizebox{\linewidth}{!}{
\begin{tabular}{ l l  c c c cc cc c}
\toprule
\multirow{3}{*}{} & \multirow{3}{*}{\textbf{Train Arch}} & \multicolumn{8}{c}{\bf Evaluation Model} \\
\cmidrule{3-10}
   && Conv & Conv-NN & AlexNet-NN & AlexNet-IN & ResNet18-IN & ResNet18-BN & VGG11-IN & VGG11-BN\\
\midrule
DD$^\dag$~\cite{wang2018dataset,deng2022remember} & Conv-IN & 60.2 $\pm$ 0.4 & 17.8 $\pm$ 2.7 & 11.4 $\pm$ 0.1 & 38.5 $\pm$ 1.3 & 33.9 $\pm$ 1.1 & 16.0 $\pm$ 1.2 & 40.0 $\pm$ 1.7 & 21.2 $\pm$ 2.4 \\
DC~\cite{zhao2021dataset} & Conv-IN & 44.9 $\pm$ 0.5 & 31.9 $\pm$ 0.6 & 27.3 $\pm$ 1.6 & 45.5 $\pm$ 0.3 & 43.3 $\pm$ 0.6 & 32.7 $\pm$ 0.9 & 43.7 $\pm$ 0.6 & 37.3 $\pm$ 1.1\\
DSA~\cite{zhao2021datasetb} & Conv-IN & 53.2 $\pm$ 0.8 & 36.4 $\pm$ 1.5 & 34.0 $\pm$ 2.3 & 45.5 $\pm$ 0.6 & 42.3 $\pm$ 0.9 & 34.9 $\pm$ 0.5 & 43.1 $\pm$ 0.9 & 40.6 $\pm$ 0.5\\
DM~\cite{DBLP:journals/corr/abs-2110-04181} & Conv-IN & 49.2 $\pm$ 0.8 & 35.2 $\pm$ 0.5 & 34.9 $\pm$ 1.1 & 44.2 $\pm$ 0.9 & 40.2 $\pm$ 0.8 & 40.1 $\pm$ 0.8 & 41.7 $\pm$ 0.7 & 43.9 $\pm$ 0.4\\
MTT~\cite{cazenavette2022dataset} & Conv-IN &64.4 $\pm$ 0.9 & 41.6 $\pm$ 1.3 & 34.2 $\pm$ 2.6 & 51.9 $\pm$ 1.3 & 45.8 $\pm$ 1.2 & 42.9 $\pm$ 1.5 & 48.5 $\pm$ 0.8 & 45.4 $\pm$ 0.9 \\
FRePo~\cite{zhou2022dataset} & Conv-BN &65.6 $\pm$ 0.6 & 65.6 $\pm$ 0.6 & 58.2 $\pm$ 0.5 & 39.0 $\pm$ 0.3 & 47.4 $\pm$ 0.7 & 53.0 $\pm$ 1.0 & 35.0 $\pm$ 0.7 & 56.8 $\pm$ 0.6\\
\bottomrule
\end{tabular}
}
\vspace{-0.4cm}
\label{tab:cross-arch}
\end{table*}

\section{Experiments} \label{sec:experiments}
In this section, we conduct two sets of quantitative experiments, \textit{i.e.}, performance and training cost evaluation, on representative dataset distillation methods that cover three classes of primary condensation metrics, including DD~\cite{wang2018dataset}, DC~\cite{zhao2021dataset}, DSA~\cite{zhao2021datasetb}, DM~\cite{DBLP:journals/corr/abs-2110-04181}, MTT~\cite{cazenavette2022dataset}, and FRePo~\cite{zhou2022dataset}.

\subsection{Experimental Setup}
The performance of dataset distillation is mainly evaluated on the classification task. 
We follow typical settings in the area of dataset distillation, which are illustrated below:

\textbf{Datasets.} 
We adopt five datasets, including MNIST~\cite{lecun1998gradient}, Fashion-MNIST~\cite{xiao2017fashion}, CIFAR-10~\cite{krizhevsky2009learning}, CIFAR-100~\cite{krizhevsky2009learning}, and Tiny-ImageNet~\cite{le2015tiny}. 
These datasets are widely used as benchmarks in existing dataset distillation works.

\textbf{Networks.} 
We use the default ConvNet~\cite{krizhevsky2017imagenet} architecture provided by the authors, which mainly consists of multiple \texttt{Conv-ReLU-AvgPooling} blocks. 
Also, we evaluate the performance of synthetic datasets across different architectures. 
The network architectures of evaluation models are as follows:
ConvNet~\cite{krizhevsky2017imagenet} with no normalization layer, AlexNet~\cite{krizhevsky2017imagenet} with no normalization layer and instance normalization layer, ResNet~\cite{he2016deep} with instance normalization layer and batch normalization layer, and VGG~\cite{simonyan2014very} with instance normalization layer and batch normalization layer.

\textbf{Evaluation Protocol.}
For performance evaluation, we first generate synthetic datasets through candidate methods and train target networks using these datasets.
Then, the performance of trained models is evaluated by the corresponding test set of the original dataset. 
As for training cost evaluation, all methods are evaluated under full batch training for fair comparisons. 
Moreover, all methods adopt the default data augmentation strategies the authors provided for distillation performance evaluation.
While for generalization evaluation, we adopt the DSA~\cite{zhao2021datasetb} data augmentation in the evaluation model training process for fair comparisons.
Additionally, we measure the distillation performance of condensed datasets with 1, 10, and 50 images per class (IPC) for different datasets. 
The cross-architecture experiments are conducted on CIFAR-10 with 10 IPC. 
Efficiency evaluation experiments are conducted on the CIFAR-10 dataset with a wide range of IPC.

\textbf{Implementation Details.}
Here, we implement DC, DSA, DM, MTT, and FRePo by the official code provided by the authors. 
For DD, we modify the inner loop optimization by adding the momentum term, which is demonstrated to improve the performance substantially~\cite{deng2022remember}. 
FRePo is implemented in JAX and PyTorch, respectively, and the rest methods are implemented in PyTorch. 
For fair comparisons, all efficiency evaluation experiments are conducted on the single A100-SXM4-40GB.

\subsection{Performance Evaluation}

\textbf{Distillation Performance.}
We evaluate the distillation performance of different dataset distillation methods on the models with the same architecture as the default training models. 
The experiment settings are default as the authors provided. 
The testing results are shown in Table~\ref{tab:performance}. 
Please refer to the supplementary for the performance of all existing DD methods. 
Due to the limitation of GPU memory, some experimental results for Tiny-ImageNet are missing. 
As shown in Table~\ref{tab:performance}, in most cases, FRePo achieves state-of-the-art performance, especially for more complex datasets, \textit{e.g.}, CIFAR-10, CIFAR-100, and Tiny-ImageNet. 
At the same time, MTT often achieves second-best performance. 
Comparing the performance results of DC and DSA, although dataset augmentation cannot guarantee to benefit the distillation performance, it influences the results significantly in most times. 
As for DD, it obtains the SOTA performance in the case of MNIST 1 IPC, Fashion-MNIST 1 IPC, and 10 IPC. 
The performance of DM is often not good as other methods.

\textbf{Cross Architecture Generalization.}
The transferability of different methods is evaluated on a wide range of network architectures, unseen during training in the case of CIFAR-10 with 10 IPC. 
From results shown in Table \ref{tab:cross-arch}, the instance normalization layer seems to be a vital ingredient in several methods (DD, DC, DSA, DM, and MTT), which may be harmful to the transferability. 
The performance degrades significantly for most methods except FRePo when no normalization is adopted (Conv-NN, AlexNet-NN). 
If the normalization layer is inconsistent with the training architecture, the test results will drop significantly. 
It suggests that the synthetic dataset generated by those methods encodes the inductive bias of a particular training architecture. 
As for DD, the distilled data highly rely on the training model, and thus the performance degrades significantly on heterogeneous networks with different normalization layers, \textit{e.g.}, batch normalization.
Moreover, Comparing the transferability of DC and DSA, we find that DSA data augmentation can help train models with different architectures, especially those with batch normalization layers.

\begin{figure*}[ht]
  \centering
  \subfigure[]{
 \includegraphics[width=0.31\linewidth]{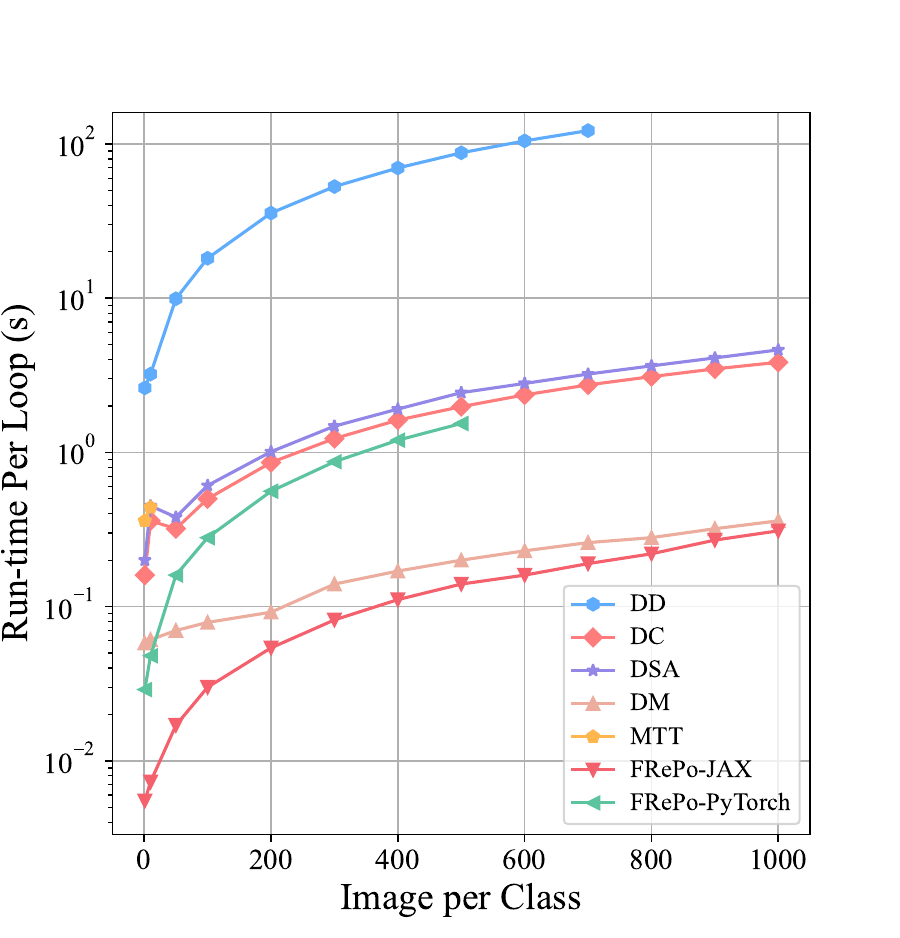}
  }
  \subfigure[]{
  \includegraphics[width=0.31\linewidth]{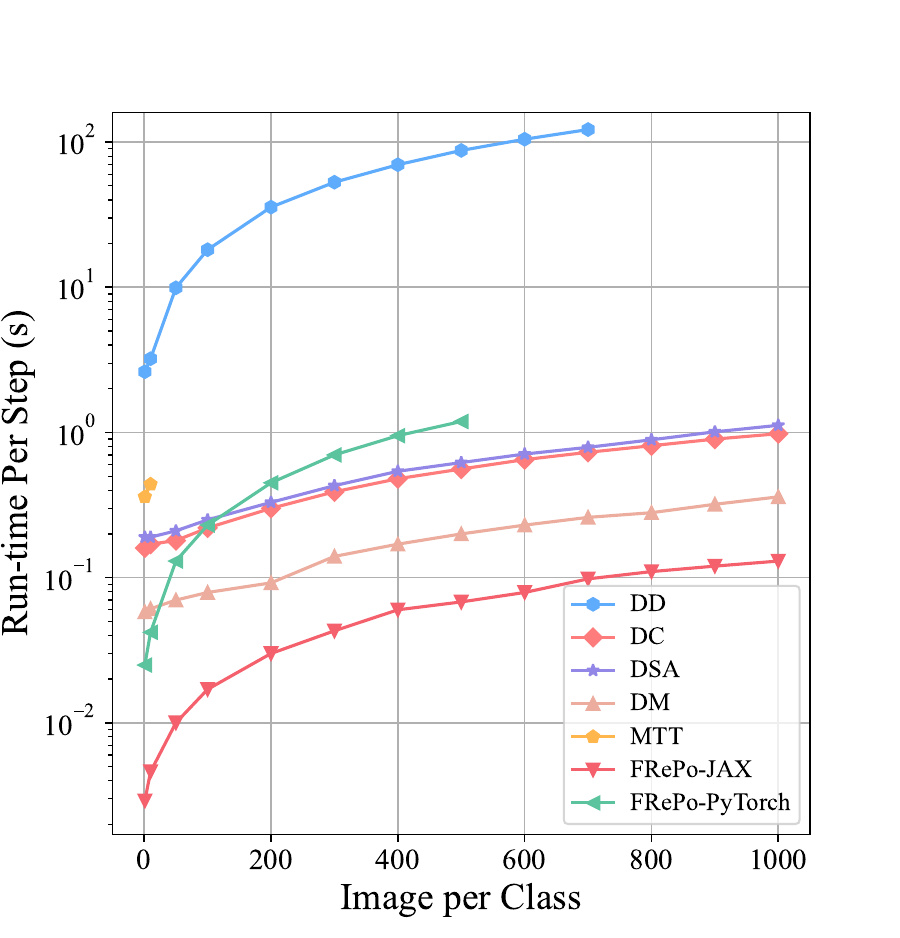} 
  }
  \subfigure[]{
 \includegraphics[width=0.31\linewidth]{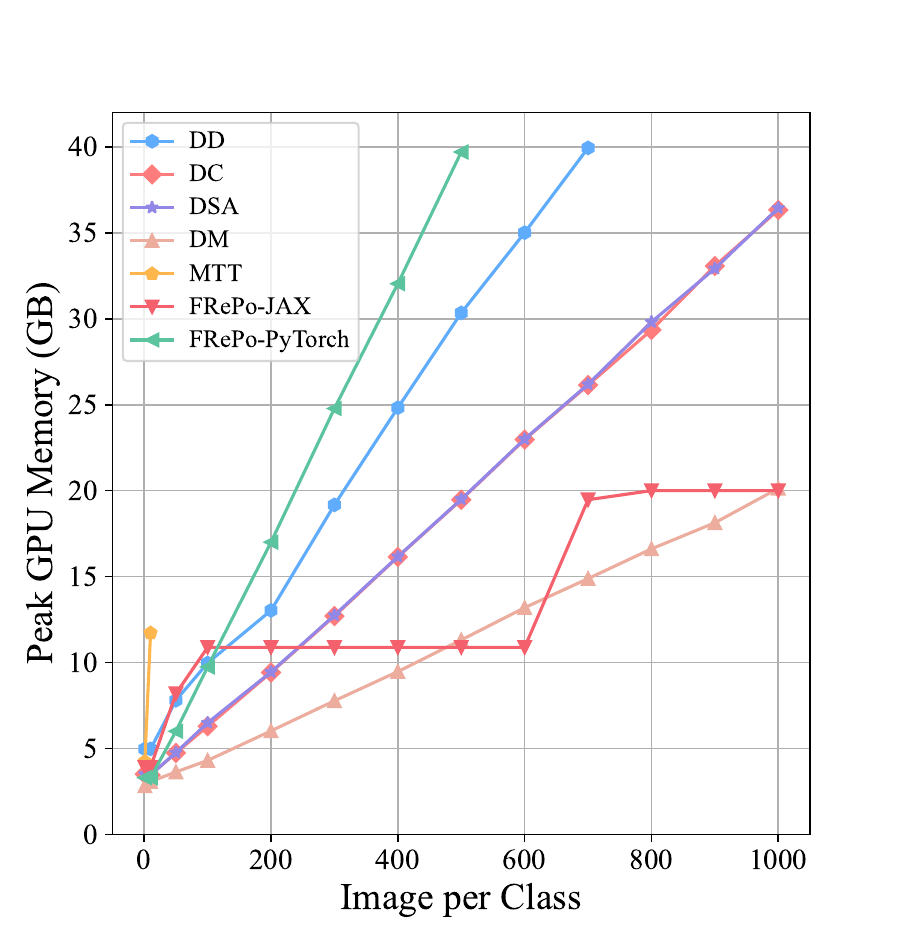}
  }
  \vspace{-0.4cm}
  \caption{(a) Run-time per loop (including time required for updating synthetic data and networks in one loop). (b) Run-time per step (including only time required for updating synthetic data). (c) Peak GPU Memory Usage. All evaluated under full batch training on CIFAR-10.}
  \vspace{-0.4cm}
  \label{fig:runtime}
\end{figure*}

\subsection{Training Cost Evaluation}
We measure run-time and peak GPU memory required by selected methods for training cost evaluation. 
For fair comparisons, in this section, the number of inner loops, if there have any, is based on the experiment settings provided by the author, 
and evaluation is under full batch training.

\textbf{Run-Time Evaluation.} 
We evaluate the required time for one outer loop and the time to calculate the loss function for a batch of synthetic data, respectively. 
As for DD and MTT, the process of network updating is unrolled, and thus the required time for the entire outer loop and the loss function is the same.
Also, as DM does not update the network, the required time for these two processes is the same. 
The evaluation results are shown in Fig.~\ref{fig:runtime}(a)(b). 
The results show that DD requires a significantly longer run time to update the synthetic dataset for a single iteration, as the gradient computation for the performance matching loss over the synthetic dataset involves bi-level optimization. 
DC, DSA, and FRePo implemented by PyTorch have similar required run times. 
Comparing DSA and DC, using DSA augmentation makes the running time slightly longer, but it can significantly improve the performance. 
When the IPC is small, FRePo (both the JAX and PyTorch versions) is significantly faster than other methods, but as the IPC increases, the PyTorch version is similar to the second-efficient echelon methods, while the JAX version is close to DM. 
MTT runs the second slowest, but there is no more data to analyze due to out-of-memory.


\textbf{Peak GPU Memory Usage.} 
We evaluate the peak GPU memory usage during the whole dataset distillation process. 
Here, we record the peak GPU memory in practice and show the results in Fig.~\ref{fig:runtime}(c). 
The results show that DM requires the lowest GPU memory, while MTT requires the most. 
When IPC is only 50, the operation of MTT encounters out-of-memory. 
In addition, with the gradual growth of IPC, JAX version FRePo reveals the advantages of memory efficiency. 
The DD and PyTorch versions of FRePo require relatively large GPU memory. 
The former needs to go through the entire training graph during backpropagation, while the latter adopts a wider network.
Moreover, there is no significant difference in the GPU memory required by DC and DSA, indicating that DSA augmentation is memory-friendly.

\subsection{Empirical Studies}
The following conclusions are drawn from the above experiments:
\begin{itemize}
    \item The performance of DD has significantly improved thanks to the momentum item. However, DD has relatively poor generalizability, and the required run-time and GPU memory are considerable. 
    \item Compared with DC and DSA, adopting DSA can significantly improve the distillation performance and generalizability while not increasing the running time and memory requirement too much. 
    \item DM does not perform as well as other methods. However, DM has relatively good generalizability and significant advantages regarding run time and GPU memory requirements.
    \item MTT has the overall second-best performance but requires much running time and space due to unrolled gradient computation through backpropagation. 
    \item FRePo achieves SOTA performance when IPC are small, regardless of performance or training cost. However, as IPC increases, FRePo does not have comparable performance.
\end{itemize}

\section{Challenges and Possible Improvements}\label{sec:challenges}
The research on dataset distillation has promising prospects, and many algorithms have been applied in various fields. 
Although existing methods have achieved reasonable performance, there are still some challenges and issues. 
In this section, we summarise the key challenges and discuss possible directions for improvements.

\subsection{Computational Cost}\label{sec:computation}
The motivation of DD is to train a small informative dataset so that the model training on it can achieve the same performance as training on the original data set, thereby reducing GPU time.
However, the problem is that generating the synthetic dataset is typically expensive and that the required time will rapidly increase when distilling large-scale datasets. 
This high computational cost is mainly caused by backpropagating through unrolled computational graphs for updating synthetic datasets. 
For example, MTT~\cite{cazenavette2022dataset} does not scale well in terms of both memory and run time issues, where unrolling 30 steps on CIFAR10 with IPC 50 requires 47GB GPU memory, and is prohibitive in large-scale problems like ImageNet-1K~\cite{russakovsky2015imagenet}, as shown in Table~\ref{tab:imagenet1k}.

To tackle this problem, many works have been proposed for different DD frameworks to reduce computational costs. 
One intuitive way is to avoid the inner-loop network training.
As for performance matching, KRR-based methods~\cite{nguyen2020dataset,nguyen2021dataset,zhang2022accelerating,zhou2022dataset,loo2022efficient} are proposed to improve the computational efficiency of DD~\cite{wang2018dataset}, which performs convex optimization and results in a closed-form solution for linear models and avoids extensive inner-loop training. 
As for distribution matching, DM~\cite{DBLP:journals/corr/abs-2110-04181} improves efficiency using random networks as the feature extractors without training, which avoids expensive bi-level optimization. 
However, the high efficiency sacrificed its performance. 

As for MTT, Cui \textit{et al.}~\cite{cui2022scaling} re-organize the computing flow for the gradients, which reduces the memory cost to constant complexity. 
Nevertheless, it still requires pretraining hundreds of teacher models on original datasets, which affects the time and memory efficiency negatively. 
Model augmentation methods like Zhang \textit{et al.}~\cite{zhang2022accelerating} may potentially be helpful to alleviate the problem.  


\subsection{Scaling Up}
The existing dataset distillation methods mainly evaluate the performance of the synthetic dataset up to about 50 IPC. 
In practical applications, it may be necessary to increase the compression ratio to obtain more information on the whole dataset, and the dataset and the network adopted may be larger and more complex. 
Under such circumstances, most existing dataset distillation methods cannot meet the demand. 
In short, the scaling-up problem of dataset distillation is reflected in three aspects: hard to obtain an informative synthetic dataset with a larger IPC; challenging to perform the dataset distillation algorithm on a large-scale dataset such as ImageNet; not easy to adapt to a large, complex model.

\textbf{Larger Compression Ratios.} 
As shown in Fig. \ref{fig:ratio}, on the whole, with the increase of IPC, the accuracy of the dataset distillation methods improves more slowly. 
Moreover, only when it is less than 200 IPC, the methods of dataset distillation are obviously better than selection-based methods. 
With the further increase of IPC, the margin between the dataset distillation methods and selection-based methods narrows considerably, and finally, the selection-based methods are slightly better than the dataset distillation methods. 
The above observation shows that as the synthetic dataset increases, the ability of the dataset to compress knowledge of the real dataset does not increase significantly, even though the memory efficiency of some methods can support larger IPC. 
In other words, many existing dataset distillation methods are not suitable for high compression ratios.

\begin{figure}[t]
\centering
\includegraphics[width=0.8\linewidth]{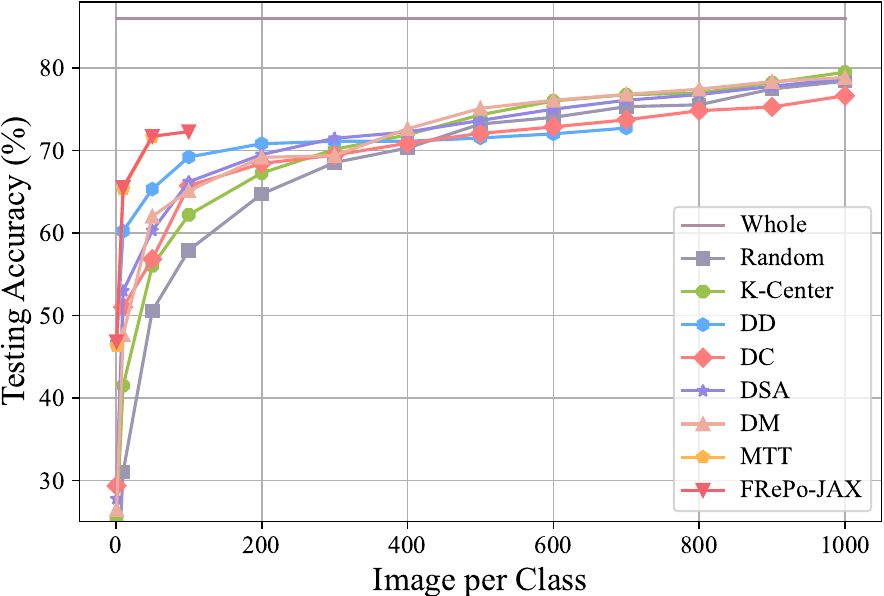}
\vspace{-0.3cm}
\caption{Performance comparison with different compression ratios on CIFAR10 using Random selection, K-Center, DD, DC, DSA, DM, MTT and FRePo-JAX. Results are derived from \cite{cui2022dc} and our experiments.}
\vspace{-0.6cm}
\label{fig:ratio}
\end{figure}

\textbf{Larger Original Datasets.} 
Another scaling-up problem is that many existing dataset distillation methods are challenging to apply to large-scale datasets due to high GPU memory requirements and extensive GPU hours as discussed in Section~\ref{sec:computation}. 
As shown in Table~\ref{tab:imagenet1k}, only a few methods can handle the DD task on ImageNet1K. 
In addition, under the same compression ratio, for the more complex original dataset, the synthetic dataset obtained by the dataset distillation method may have a weaker ability to restore the performance of the original dataset, \textit{i.e.}, the ratio of test accuracy of models trained on $\mathcal{S}$ and $\mathcal{T}$.

\textbf{Larger Networks.}
Many dataset distillation methods include the process of bi-level meta-optimization. 
Larger and more complex network structures require more GPU memory and runtime to compute and store the gradients during the dataset distillation process.
At the same time, as shown in Table \ref{tab:largenetwork}, when the generated dataset is evaluated on complex networks, its performance is significantly lower than that of simple networks. 
As the size of the synthetic dataset is small, complex networks are prone to overfitting problems trained on the small training dataset.

\begin{table}[t]
\renewcommand{\arraystretch}{0.9}
\tiny
\centering
\caption{The performance results of different dataset distillation methods on ImageNet-1K. Results are derived from~\cite{cui2022scaling} and our experiments. $\dag$ adopt momentum~\cite{deng2022remember}.}
\vspace{-0.2cm}
\resizebox{\linewidth}{!}{
\begin{tabular}{cl c  ccc }
\toprule
&& \multicolumn{4}{c}{\bf IPC} \\
\cmidrule{3-6}
&&1 & 2 & 10 &50   \\
\midrule
\multirow{7}{*}{\rotatebox[origin=c]{90}{\bf ImageNet-1K}}
&DD$^\dag$~\cite{wang2018dataset,deng2022remember} & -& - & - & -\\
&DC~\cite{zhao2021dataset} & - & - & - & -  \\
&DSA~\cite{zhao2021datasetb} & - & - & - & -  \\
&DM~\cite{DBLP:journals/corr/abs-2110-04181} & 1.5 $\pm$ 0.1 & 1.7 $\pm$ 0.1 & - & -  \\
&MTT~\cite{cazenavette2022dataset} & - & - & - & -  \\
&FRePo~\cite{zhou2022dataset} & 7.5 $\pm$ 0.3 & 9.7 $\pm$ 0.2 & - & -  \\
&TESLA~\cite{cui2022scaling} & 7.7 $\pm$ 0.2 & 10.5 $\pm$ 0.3 & 17.8 $\pm$ 1.3 & 27.9 $\pm$ 1.2 \\
\bottomrule
\end{tabular}
}
\vspace{-0.2cm}
\label{tab:imagenet1k}
\end{table}

\begin{table}[t]
\renewcommand{\arraystretch}{0.9}
\centering
\large
\caption{The evaluation performance results on CIFAR-10 with 10 Img/Cls for more extensive networks. $\dag$ adopt momentum~\cite{deng2022remember}.}
\vspace{-0.2cm}
\resizebox{\linewidth}{!}{
\begin{tabular}{l c  cccc }
\toprule
& \multicolumn{5}{c}{\bf Evaluation Model} \\
\cmidrule{2-6}
& ResNet18 & ResNet34 & ResNet50 & ResNet101 &ResNet152  \\
\midrule

DD$^\dag$~\cite{wang2018dataset,deng2022remember} & 33.9 $\pm$ 1.1 & 32.2 $\pm$ 2.7 & 19.0 $\pm$ 2.0 & 18.2 $\pm$ 1.4 & 12.3$\pm$ 0.7\\
DC~\cite{zhao2021dataset} & 43.3 $\pm$ 0.6 & 35.4 $\pm$ 1.3 & 23.7 $\pm$ 0.4 & 17.6 $\pm$ 1.0 & 16.3 $\pm$ 1.0 \\
DSA~\cite{zhao2021datasetb} & 42.3 $\pm$ 0.9 & 33.2 $\pm$ 1.0 & 22.9 $\pm$ 0.7 & 18.7 $\pm$ 1.3 & 15.7 $\pm$ 1.0 \\
DM~\cite{DBLP:journals/corr/abs-2110-04181} & 39.5 $\pm$ 0.6 & 31.2 $\pm$ 1.1 & 23.9 $\pm$ 0.6 & 17.7 $\pm$ 1.3 & 16.8 $\pm$ 1.4 \\
MTT~\cite{cazenavette2022dataset} & 45.8 $\pm$ 1.2 & 34.6 $\pm$ 3.3 & 22.5 $\pm$ 0.4 & 18.1 $\pm$ 2.0 & 18.0 $\pm$ 1.1 \\
FRePo~\cite{zhou2022dataset} & 47.4 $\pm$ 0.7 & 41.3 $\pm$ 2.0 & 41.1 $\pm$ 1.8 & 28.7 $\pm$ 1.2 & 22.7 $\pm$ 1.0 \\
\bottomrule
\end{tabular}
}
\vspace{-0.5cm}
\label{tab:largenetwork}
\end{table}

\subsection{Generalization across Different Architectures}
The generalizability of the synthetic dataset across different network architectures is an essential standard to evaluate the performance of a DD method, as it indicates availability in practice. 
However, the evaluation results of the synthetic data generated by the existing DD methods on heterogenous models have not reached homogeneous performance. 
Through the results shown in Table \ref{tab:cross-arch}, if the synthetic dataset and the training architecture are strongly bonded in the optimization process, it will hurt the transferability of the generated synthetic dataset. 
Moreover, the choice of the normalization layer significantly impacts transferability. 

To remedy this problem, one intuitive idea is to unbind the network and the synthetic dataset, such as the practice of DM~\cite{DBLP:journals/corr/abs-2110-04181} and IDC~\cite{kim2022dataset}. 
DM treats the network as the feature extractor without updating, and IDC updates the training network on the real dataset instead of the synthetic dataset. 
From the results, these strategies are somehow effective. 
However, DM has relatively poor performance, and IDC needs extra run time to train networks on real datasets, especially when distilling large-scale datasets like ImageNet. 
Another way is to increase the diversity of training architectures, where dedicated schemes are necessary to coordinate their training and their impact on synthetic data. 

\subsection{Design for Other Tasks and Applications}
The existing DD methods mainly focus on the classification task. Here, we expect future works to apply DD and conduct sophisticated designing in more tasks in computer vision, \textit{e.g.}, semantic segmentation~\cite{long2015fully,chen2017deeplab} and objective detection~\cite{girshick2014rich}, natural language processing, \textit{e.g.}, machine translation~\cite{bahdanau2014neural}, and multi-modal scenarios~\cite{radford2021learning}. 

\subsection{Security and Privacy}
Existing works focus on improving the methods to generate more informative synthetic datasets while neglecting the potential security and privacy of DD. 
A recent work~\cite{liu2023backdoor} provides a new backdoor attack method, DOORPING, which attacks during the dataset distillation process rather than afterward model training as introduced in Section \ref{sec:privacy}. 
Possible defense mechanisms and more security and privacy issues are left for future work.

\section{Conclusion}\label{sec:conclusion}
This paper aims at a comprehensive review of dataset distillation (DD), a popular research topic recently, to synthesize a small dataset given an original large one for the sake of similar performance. 
We present a systematic taxonomy of existing DD methods and categorize them into three major streams by the optimization objective: performance matching, parameter matching, and distribution matching. 
Theoretical analysis reveals their underlying connections. 
For general understanding, we abstract a common algorithmic framework for all current approaches. 
Application of DD on various topics, including continual learning, neural architecture search, privacy, \textit{etc} is also covered. 
Different approaches are experimentally compared in terms of accuracy, time efficiency, and scalability, indicating some critical challenges in this area left for future research. 


%




\section*{Acknowledgment}
We sincerely thank Guang Li \textit{et al.} for the useful information in \href{https://github.com/Guang000/Awesome-Dataset-Distillation}{Awesome-Dataset-Distillation}.


\ifCLASSOPTIONcaptionsoff
  \newpage
\fi



\bibliographystyle{IEEEtran}
\bibliography{IEEEabrv,main}
%



%

\appendices
\section{Taxonomy of Existing DD Methods}
In Table~\ref{tab:summary}, we provide a tablular version of the taxonomy shown in Fig. 5 of the main manuscript, to provide an alternative view with more details. 
A DD method can be analyzed from 4 aspects: optimization objective, fashion of updating networks, synthetic data parameterization, and fashion of learning labels. 

\section{Performance of Existing DD Methods}
Here, we list the distillation performance of all existing DD methods on benchmark datasets with different IPC in Table \ref{tab:mnist} - Table \ref{tab:imagenet-1k}. The benchmark datasets include MNIST~\cite{lecun1998gradient}, Fashion-MNIST~\cite{xiao2017fashion}, SVHN~\cite{netzer2011reading}, CIFAR-10~\cite{krizhevsky2009learning}, CIFAR-100~\cite{krizhevsky2009learning}, Tiny-ImageNet~\cite{le2015tiny} and ImageNet-1K~\cite{russakovsky2015imagenet}. Results are derived from official papers, DC-Benchmark~\cite{cui2022dc} and our experiments.  

\begin{table*}[ht]
\tiny
\renewcommand{\arraystretch}{1.2}
\centering
\caption{Taxonomy of existing dataset distillation methods. 'v' denotes some variant of an optimization objective here.}
\resizebox{\linewidth}{!}{
\begin{tabularx}{13cm}{l| XXX l}
\toprule
 \bf Method & \bf Objective & \bf Network Update & \bf Parameterization & \bf Label \\
\midrule
DD~\cite{wang2018dataset} & Performance Matching-Meta & Use $\mathcal{S}$ (Unroll) & Raw & Fixed\\
AddMem~\cite{deng2022remember} & Performance Matching-Meta & Use $\mathcal{S}$ (Unroll) & Memory-Addressing Matrix & Learnable\\
KIP~\cite{nguyen2020dataset,nguyen2021dataset} & Performance Matching-KRR & Do not Update Network & Raw & Learnable\\
FRePo~\cite{zhou2022dataset} & Performance Matching-KRR & Use $\mathcal{S}$ with a Network Pool & Raw & Learnable \\
RFAD~\cite{loo2022efficient} & Performance Matching-KRR & Do not Update Network & Raw & Learnable\\
\cmidrule{1-5}
DC~\cite{zhao2021dataset} & Single-Step Parameter Matching & Use $\mathcal{S}$ & Raw & Fixed\\
DSA~\cite{zhao2021datasetb} & Single-Step Parameter Matching & Use $\mathcal{S}$ & DSA & Fixed\\
IDC~\cite{kim2022dataset} & Single-Step Parameter Matching & Use $\mathcal{T}$ & Up-sampling \& DSA & Fixed\\
EDD~\cite{zhang2022accelerating} & Single-Step Parameter Matching & Use $\mathcal{T}$ and Model Augmentation & Up-sampling \& DSA & Fixed\\
DCC~\cite{lee2022contrastive} & v - Single-Step Parameter Matching & Use $\mathcal{S}$ & DSA / Raw & Fixed\\
EGM~\cite{jiang2022delving} & v - Single-Step Parameter Matching & Use $\mathcal{S}$ & Raw & Fixed\\
MTT~\cite{cazenavette2022dataset} & Multi-Step Parameter Matching & Use $\mathcal{T}$ (Cached) and $\mathcal{S}$ (Unroll)& DSA & Fixed\\
DDPP~\cite{li2022datasetb} & Multi-Step Parameter Matching & Use $\mathcal{T}$ (Cached),  $\mathcal{S}$ (Unroll) and Parameter Pruning & DSA & Fixed\\
HaBa~\cite{liu2022dataset} & Multi-Step Parameter Matching & Use $\mathcal{T}$ (Cached) and $\mathcal{S}$ (Unroll) & Hallucinator-Basis \& DSA & Fixed\\
FTD~\cite{du2022minimizing} & Multi-Step Parameter Matching & Use $\mathcal{T}$ with Flat Regularization (Cached) and $\mathcal{S}$ (Unroll) & DSA & Fixed\\
TESLA~\cite{cui2022scaling} & v - Multi-Step Parameter Matching & Use $\mathcal{T}$ (Cached) and $\mathcal{S}$ & DSA & Soft\\
\cmidrule{1-5}
DM~\cite{DBLP:journals/corr/abs-2110-04181} & Single-Layer Distribution Matching & Do not Update Network & DSA & Fixed\\
IT-GAN~\cite{zhao2022synthesizing} & Single-Layer Distribution Matching & Do not Update Network & Latent Code-GAN & Fixed\\
KFS~\cite{lee2022dataset} & Single-Layer Distribution Matching & Do not Update Network & Latent Code-Decoder & Fixed\\
CAFE~\cite{wang2022cafe} & Multi-Layer Distribution Matching & Use $\mathcal{S}$ & Raw / DSA & Fixed\\
\bottomrule
\end{tabularx}
}
\label{tab:summary}
\end{table*}

\begin{table}[H]
\renewcommand{\arraystretch}{1.1}
\centering
\caption{The performance results of all existing methods on MNIST. $\dag$ adopt momentum~\cite{deng2022remember}.}
\resizebox{\linewidth}{!}{
\begin{tabular}{cl c  cc }
\toprule
&& \multicolumn{3}{c}{\bf IPC} \\
\cmidrule{3-5}
&&1 & 10 &50   \\
\midrule
\multirow{25}{*}{\rotatebox[origin=c]{90}{\bf MNIST}}
&DD~\cite{wang2018dataset} & - & 79.5 $\pm$ 8.1 & -\\
&DD$^\dag$~\cite{wang2018dataset,deng2022remember} & 95.2 $\pm$ 0.3 & 98.0 $\pm$ 0.1 & 98.8 $\pm$ 0.1\\
&AddMem~\cite{deng2022remember} & 98.7 $\pm$ 0.7 & 99.3 $\pm$ 0.5 & 99.4 $\pm$ 0.4\\
&KIP~(ConvNet)~\cite{nguyen2021dataset} & 90.1 $\pm$ 0.1 & 97.5 $\pm$ 0.0 & 98.3 $\pm$ 0.1\\
&KIP~($\infty$-FC)~\cite{nguyen2020dataset}& 85.5 $\pm$ 0.1 & 97.2 $\pm$ 0.2 & 98.4 $\pm$ 0.1\\
&KIP~($\infty$-Conv)~\cite{nguyen2021dataset}& 97.3 $\pm$ 0.1 & 99.1 $\pm$ 0.1 & 99.5 $\pm$ 0.1 \\
&FRePo~(ConvNet)~\cite{zhou2022dataset} & 93.8 $\pm$ 0.6 & 98.4 $\pm$ 0.1 & 99.2 $\pm$ 0.1 \\
&FRePo~($\infty$-Conv)~\cite{zhou2022dataset}& 93.6 $\pm$ 0.5 & 98.5 $\pm$ 0.1 & 99.1 $\pm$ 0.1 \\
&RFAD~(ConvNet)~\cite{loo2022efficient} & 94.4 $\pm$ 1.5 & 98.5 $\pm$ 0.1 & 98.8 $\pm$ 0.1\\
&RFAD~($\infty$-Conv)~\cite{loo2022efficient} & 97.2 $\pm$ 0.2 & 99.1 $\pm$ 0.0 & 99.1 $\pm$ 0.0\\
\cmidrule{2-5}
&DC~\cite{zhao2021dataset} & 91.7 $\pm$ 0.5 & 97.4 $\pm$ 0.2 & 98.8 $\pm$ 0.2\\
&DSA~\cite{zhao2021datasetb} &88.7 $\pm$ 0.6 & 97.9 $\pm$ 0.1 & 99.2 $\pm$ 0.1 \\
&IDC~\cite{kim2022dataset} & -&-&-\\
&EDD~\cite{zhang2022accelerating} &-&-&- \\
&DCC~\cite{lee2022contrastive} & - &- &-\\
&EGM~\cite{jiang2022delving} &91.9 $\pm$ 0.4 & 97.9 $\pm$ 0.2 & 98.6 $\pm$ 0.1\\
&MTT~\cite{cazenavette2022dataset} & 91.4 $\pm$ 0.9 & 97.3 $\pm$ 0.1 & 98.5 $\pm$ 0.1\\
&DDPP~\cite{li2022datasetb} &-&-&- \\
&HaBa~\cite{liu2022dataset} &-&-&- \\
&FTD~\cite{du2022minimizing} &-&-&- \\
&TESLA~\cite{cui2022scaling} &-&-&- \\
\cmidrule{2-5}
&DM~\cite{DBLP:journals/corr/abs-2110-04181} &89.9 $\pm$ 0.8&97.6 $\pm$ 0.1 &98.6 $\pm$ 0.1 \\
&IT-GAN~\cite{zhao2022synthesizing} & - &- &-\\
&KFS~\cite{lee2022dataset} & - &- &-\\
&CAFE~\cite{wang2022cafe} & 93.1 $\pm$ 0.3 & 97.5 $\pm$ 0.1 & 98.9 $\pm$ 0.2\\
\bottomrule
\end{tabular}
}
\label{tab:mnist}
\end{table}

\begin{table}[H]
\renewcommand{\arraystretch}{1.1}
\centering
\caption{The performance results of all existing methods on Fashion-MNIST. $\dag$ adopt momentum~\cite{deng2022remember}.}
\resizebox{\linewidth}{!}{
\begin{tabular}{cl c  cc }
\toprule
&& \multicolumn{3}{c}{\bf IPC} \\
\cmidrule{3-5}
&&1 & 10 &50   \\
\midrule
\multirow{25}{*}{\rotatebox[origin=c]{90}{\bf Fashion-MNIST}}
&DD~\cite{wang2018dataset} & - & - & -\\
&DD$^\dag$~\cite{wang2018dataset,deng2022remember} & 83.4 $\pm$ 0.3 & 87.6 $\pm$ 0.4 & 87.7 $\pm$ 0.3\\
&AddMem~\cite{deng2022remember} & 88.5 $\pm$ 0.1 & 90.0 $\pm$ 0.7 & 91.2 $\pm$ 0.3\\
&KIP~(ConvNet)~\cite{nguyen2021dataset} & 70.6 $\pm$ 0.6 & 84.6 $\pm$ 0.3 & 88.7 $\pm$ 0.2\\
&KIP~($\infty$-FC)~\cite{nguyen2020dataset}&- & - & -\\
&KIP~($\infty$-Conv)~\cite{nguyen2021dataset}& 82.9 $\pm$ 0.2 & 91.0 $\pm$ 0.1 & 92.4 $\pm$ 0.1 \\
&FRePo~(ConvNet)~\cite{zhou2022dataset} & 75.6 $\pm$ 0.5 & 86.2 $\pm$ 0.3 & 89.6 $\pm$ 0.1 \\
&FRePo~($\infty$-Conv)~\cite{zhou2022dataset}& 76.3 $\pm$ 0.4 & 86.7 $\pm$ 0.3 & 90.0$\pm$0.0 \\
&RFAD~(ConvNet)~\cite{loo2022efficient} & 78.6 $\pm$ 1.3 & 87.0 $\pm$ 0.5 & 88.8 $\pm$ 0.4\\
&RFAD~($\infty$-Conv)~\cite{loo2022efficient} & 84.6 $\pm$ 0.2 & 90.3 $\pm$ 0.2 & 91.4 $\pm$ 0.1\\
\cmidrule{2-5}
&DC~\cite{zhao2021dataset} & 70.5 $\pm$ 0.6 & 82.3 $\pm$ 0.4 & 83.6 $\pm$ 0.4\\
&DSA~\cite{zhao2021datasetb} &70.6 $\pm$ 0.6 & 84.8 $\pm$ 0.3 & 88.8 $\pm$ 0.2 \\
&IDC~\cite{kim2022dataset} & -&-&-\\
&EDD~\cite{zhang2022accelerating} &-&-&- \\
&DCC~\cite{lee2022contrastive} & - &- &-\\
&EGM~\cite{jiang2022delving} &71.4 $\pm$ 0.4 & 85.4 $\pm$ 0.3 & 87.9 $\pm$ 0.2\\
&MTT~\cite{cazenavette2022dataset} & 75.1 $\pm$ 0.9 &87.2 $\pm$ 0.3 & 88.3 $\pm$ 0.1 \\
&DDPP~\cite{li2022datasetb} &-&-&- \\
&HaBa~\cite{liu2022dataset} &-&-&- \\
&FTD~\cite{du2022minimizing} &-&-&- \\
&TESLA~\cite{cui2022scaling} &-&-&- \\
\cmidrule{2-5}
&DM~\cite{DBLP:journals/corr/abs-2110-04181} &71.5 $\pm$ 0.5&83.6 $\pm$ 0.2 &88.2 $\pm$ 0.1 \\
&IT-GAN~\cite{zhao2022synthesizing} & - &- &-\\
&KFS~\cite{lee2022dataset} & - &- &-\\
&CAFE~\cite{wang2022cafe} & 73.7 $\pm$ 0.7 & 83.0 $\pm$ 0.4 & 88.2 $\pm$ 0.3\\
\bottomrule
\end{tabular}
}
\label{tab:FashionMNIST}
\end{table}

\clearpage

\begin{table}[H]
\renewcommand{\arraystretch}{1.1}
\centering
\caption{The performance results of all existing methods on SVHN. $\dag$ adopt momentum~\cite{deng2022remember}.}
\resizebox{\linewidth}{!}{
\begin{tabular}{cl c  cc }
\toprule
&& \multicolumn{3}{c}{\bf IPC} \\
\cmidrule{3-5}
&&1 & 10 &50   \\
\midrule
\multirow{25}{*}{\rotatebox[origin=c]{90}{\bf SVHN}}
&DD~\cite{wang2018dataset} & - & - & -\\
&DD$^\dag$~\cite{wang2018dataset,deng2022remember} & - & - & -\\
&AddMem~\cite{deng2022remember} & 87.3 $\pm$ 0.1 & 89.1 $\pm$ 0.2 & 89.5 $\pm$ 0.2 \\
&KIP~(ConvNet)~\cite{nguyen2021dataset} & 57.3 $\pm$ 0.1 & 75.0 $\pm$ 0.1 & 80.5 $\pm$ 0.1 \\
&KIP~($\infty$-FC)~\cite{nguyen2020dataset}& - & - & -\\
&KIP~($\infty$-Conv)~\cite{nguyen2021dataset}& 64.3 $\pm$ 0.4 & 81.1 $\pm$ 0.5 & 84.3 $\pm$ 0.1 \\
&FRePo~(ConvNet)~\cite{zhou2022dataset} & - & -& - \\
&FRePo~($\infty$-Conv)~\cite{zhou2022dataset}& - & - & - \\
&RFAD~(ConvNet)~\cite{loo2022efficient} & 52.2 $\pm$ 2.2 & 74.9 $\pm$ 0.4 & 80.9 $\pm$ 0.3 \\
&RFAD~($\infty$-Conv)~\cite{loo2022efficient} & 57.4 $\pm$ 0.8 & 78.2 $\pm$ 0.5 & 82.4 $\pm$ 0.1 \\
\cmidrule{2-5}
&DC~\cite{zhao2021dataset} & 31.2 $\pm$ 1.4 & 76.1 $\pm$ 0.6 & 82.3 $\pm$ 0.3 \\
&DSA~\cite{zhao2021datasetb} &27.5 $\pm$ 1.4 & 79.2 $\pm$ 0.5 & 84.4 $\pm$ 0.4 \\
&IDC~\cite{kim2022dataset} & -&-&-\\
&EDD~\cite{zhang2022accelerating} &-&-&- \\
&DCC~\cite{lee2022contrastive} & 47.5 $\pm$ 2.6 & 80.5 $\pm$ 0.6 & 87.2 $\pm$ 0.3 \\
&EGM~\cite{jiang2022delving} & 34.5 $\pm$ 1.9 & 76.2 $\pm$ 0.7 & 83.8 $\pm$ 0.3 \\
&MTT~\cite{cazenavette2022dataset} & - & - & -\\
&DDPP~\cite{li2022datasetb} &-&-&- \\
&HaBa~\cite{liu2022dataset} & 69.8 $\pm$ 1.3 & 83.2 $\pm$ 0.4 & 88.3 $\pm$ 0.1 \\
&FTD~\cite{du2022minimizing} &-&-&- \\
&TESLA~\cite{cui2022scaling} &-&-&- \\
\cmidrule{2-5}
&DM~\cite{DBLP:journals/corr/abs-2110-04181} & - & - & -\\
&IT-GAN~\cite{zhao2022synthesizing} & - &- &-\\
&KFS~\cite{lee2022dataset} & - &- &-\\
&CAFE~\cite{wang2022cafe} & 42.9 $\pm$ 3.0 & 77.9 $\pm$ 0.6 & 82.3 $\pm$ 0.4 \\
\bottomrule
\end{tabular}
}
\label{tab:svhn}
\end{table}

\begin{table}[H]
\renewcommand{\arraystretch}{1.1}
\centering
\caption{The performance results of all existing methods on CIFAR-10. $\dag$ adopt momentum~\cite{deng2022remember}.}
\resizebox{\linewidth}{!}{
\begin{tabular}{cl c  cc }
\toprule
&& \multicolumn{3}{c}{\bf IPC} \\
\cmidrule{3-5}
&&1 & 10 &50   \\
\midrule
\multirow{25}{*}{\rotatebox[origin=c]{90}{\bf CIFAR-10}}
&DD~\cite{wang2018dataset} & - & 36.8 $\pm$ 1.2 & -\\
&DD$^\dag$~\cite{wang2018dataset,deng2022remember} & 46.6 $\pm$ 0.6 & 60.2 $\pm$ 0.4 &65.3 $\pm$ 0.4\\
&AddMem~\cite{deng2022remember} & 66.4 $\pm$ 0.4 &71.2 $\pm$ 0.4 & 73.6 $\pm$ 0.5\\
&KIP~(ConvNet)~\cite{nguyen2021dataset} & 49.9 $\pm$ 0.2 & 62.7 $\pm$ 0.3 & 68.6 $\pm$ 0.2\\
&KIP~($\infty$-FC)~\cite{nguyen2020dataset}&40.5 $\pm$ 0.4 & 53.1 $\pm$ 0.5 & 58.6 $\pm$ 0.4\\
&KIP~($\infty$-Conv)~\cite{nguyen2021dataset}& 64.3 $\pm$ 0.4 & 81.1 $\pm$ 0.5 & 84.3 $\pm$ 0.1 \\
&FRePo~(ConvNet)~\cite{zhou2022dataset} & 46.8 $\pm$ 0.7 & 65.5 $\pm$ 0.6 & 71.7 $\pm$ 0.2 \\
&FRePo~($\infty$-Conv)~\cite{zhou2022dataset}& 47.9 $\pm$ 0.6 & 68.0 $\pm$ 0.2 & 74.4 $\pm$ 0.1 \\
&RFAD~(ConvNet)~\cite{loo2022efficient} & 53.6 $\pm$ 1.2& 66.3 $\pm$ 0.5 & 71.1 $\pm$ 0.4\\
&RFAD~($\infty$-Conv)~\cite{loo2022efficient} & 61.4 $\pm$ 0.8 & 73.7 $\pm$ 0.2 & 76.6 $\pm$ 0.3\\
\cmidrule{2-5}
&DC~\cite{zhao2021dataset} & 28.3 $\pm$ 0.5 & 44.9 $\pm$ 0.5 & 53.9 $\pm$ 0.5\\
&DSA~\cite{zhao2021datasetb} &28.8 $\pm$ 0.7 & 53.2 $\pm$ 0.8 & 60.6 $\pm$ 0.5\\
&IDC~\cite{kim2022dataset} & 50.0 $\pm$ 0.4 & 67.5 $\pm$ 0.5 & 74.5 $\pm$ 0.1\\
&EDD~\cite{zhang2022accelerating} &49.2 $\pm$ 0.4 & 67.1 $\pm$ 0.2 & 73.8 $\pm$ 0.1 \\
&DCC~\cite{lee2022contrastive} & 34.0 $\pm$ 0.7 & 54.5 $\pm$ 0.5 & 64.2 $\pm$ 0.4\\
&EGM~\cite{jiang2022delving} &30.0 $\pm$ 0.6 & 50.2 $\pm$ 0.6 & 60.0 $\pm$ 0.4\\
&MTT~\cite{cazenavette2022dataset} & 46.3 $\pm$ 0.8 & 65.3 $\pm$ 0.7 & 71.6 $\pm$ 0.2 \\
&DDPP~\cite{li2022datasetb} &46.4 $\pm$ 0.6 & 65.5 $\pm$ 0.3 & 71.9 $\pm$ 0.2 \\
&HaBa~\cite{liu2022dataset} &48.3 $\pm$ 0.8 & 69.9 $\pm$ 0.4 & 74.0 $\pm$ 0.2\\
&FTD~\cite{du2022minimizing} &46.8 $\pm$ 0.3 & 66.6 $\pm$ 0.3 & 73.8 $\pm$ 0.2\\
&TESLA~\cite{cui2022scaling} &48.5 $\pm$ 0.8 & 66.4 $\pm$ 0.8 & 72.6 $\pm$ 0.7\\
\cmidrule{2-5}
&DM~\cite{DBLP:journals/corr/abs-2110-04181} &26.5 $\pm$ 0.4 & 48.9 $\pm$ 0.6 & 63.0 $\pm$ 0.4 \\
&IT-GAN~\cite{zhao2022synthesizing} & - &- &-\\
&KFS~\cite{lee2022dataset} & 59.8 $\pm$ 0.5 & 72.0 $\pm$ 0.3 & 75.0 $\pm$ 0.2 \\
&CAFE~\cite{wang2022cafe} & 31.6 $\pm$ 0.8 & 50.9 $\pm$ 0.5 & 62.3 $\pm$ 0.4\\
\bottomrule
\end{tabular}
}
\label{tab:cifar10}
\end{table}

\begin{table}[H]
\renewcommand{\arraystretch}{1.1}
\centering
\caption{The performance results of all existing methods on CIFAR-100. $\dag$ adopt momentum~\cite{deng2022remember}.}
\resizebox{\linewidth}{!}{
\begin{tabular}{cl c  cc }
\toprule
&& \multicolumn{3}{c}{\bf IPC} \\
\cmidrule{3-5}
&&1 & 10 &50   \\
\midrule
\multirow{25}{*}{\rotatebox[origin=c]{90}{\bf CIFAR-100}}
&DD~\cite{wang2018dataset} & - &- & -\\
&DD$^\dag$~\cite{wang2018dataset,deng2022remember} & 19.6 $\pm$ 0.4 & 32.7 $\pm$ 0.4 & 35.0 $\pm$ 0.3\\
&AddMem~\cite{deng2022remember} & 34.0 $\pm$ 0.4 & 42.9 $\pm$ 0.7 & -\\
&KIP~(ConvNet)~\cite{nguyen2021dataset} & 15.7 $\pm$ 0.2 & 28.3 $\pm$ 0.1 & -\\
&KIP~($\infty$-FC)~\cite{nguyen2020dataset}& - & - & - \\
&KIP~($\infty$-Conv)~\cite{nguyen2021dataset}& 34.9 $\pm$ 0.1 & 49.5 $\pm$ 0.3 & - \\
&FRePo~(ConvNet)~\cite{zhou2022dataset} & 27.2 $\pm$ 0.4 & 41.3 $\pm$ 0.2 & 44.3 $\pm$ 0.2\\
&FRePo~($\infty$-Conv)~\cite{zhou2022dataset}& 30.4 $\pm$ 0.3 & 42.1 $\pm$ 0.2 & 43.0 $\pm$ 0.3\\
&RFAD~(ConvNet)~\cite{loo2022efficient} & 26.3 $\pm$ 1.1 & 33.0 $\pm$ 0.3 & -\\
&RFAD~($\infty$-Conv)~\cite{loo2022efficient} & 44.1 $\pm$ 0.1 & 46.8 $\pm$ 0.2 & -\\
\cmidrule{2-5}
&DC~\cite{zhao2021dataset} & 12.6 $\pm$ 0.4 & 25.4 $\pm$ 0.3 & 29.7 $\pm$ 0.3\\
&DSA~\cite{zhao2021datasetb} & 13.9 $\pm$ 0.3 & 32.3 $\pm$ 0.3 & 42.8 $\pm$ 0.4 \\
&IDC~\cite{kim2022dataset} & - & 44.8 $\pm$ 0.2 & -\\
&EDD~\cite{zhang2022accelerating} & 29.8 $\pm$ 0.2 & 46.2 $\pm$ 0.3 & 52.6 $\pm$ 0.4 \\
&DCC~\cite{lee2022contrastive} & 14.6 $\pm$ 0.3 & 33.5 $\pm$ 0.3 & 40.0 $\pm$ 0.3\\
&EGM~\cite{jiang2022delving} &12.7 $\pm$ 0.4 & 31.1 $\pm$ 0.3 & - \\
&MTT~\cite{cazenavette2022dataset} & 24.3 $\pm$ 0.3 & 40.1 $\pm$ 0.4 & 47.7 $\pm$ 0.2 \\
&DDPP~\cite{li2022datasetb} & 24.6 $\pm$ 0.1 & 43.1 $\pm$ 0.3 & 48.4 $\pm$ 0.3 \\
&HaBa~\cite{liu2022dataset} & 33.4 $\pm$ 0.4 & 40.2 $\pm$ 0.2 & 47.0 $\pm$ 0.2 \\
&FTD~\cite{du2022minimizing} &25.2 $\pm$ 0.2 & 43.4 $\pm$ 0.3 & 50.7 $\pm$ 0.3 \\
&TESLA~\cite{cui2022scaling} &24.8 $\pm$ 0.4 & 41.7 $\pm$ 0.3 & 47.9 $\pm$ 0.3 \\
\cmidrule{2-5}
&DM~\cite{DBLP:journals/corr/abs-2110-04181} & 11.4 $\pm$ 0.3 & 29.7 $\pm$ 0.3 & 43.6 $\pm$ 0.4 \\
&IT-GAN~\cite{zhao2022synthesizing} & - &- &-\\
&KFS~\cite{lee2022dataset} & 40.0 $\pm$ 0.5 & 50.6 $\pm$ 0.2 & - \\
&CAFE~\cite{wang2022cafe} & 14.0 $\pm$ 0.3 & 31.5 $\pm$ 0.2 & 42.9 $\pm$ 0.2\\
\bottomrule
\end{tabular}
}
\label{tab:cifar100}
\end{table}

\begin{table}[H]
\renewcommand{\arraystretch}{1.1}
\centering
\caption{The performance results of all existing methods on Tiny-ImageNet. $\dag$ adopt momentum~\cite{deng2022remember}.}
\resizebox{\linewidth}{!}{
\begin{tabular}{cl c  cc }
\toprule
&& \multicolumn{3}{c}{\bf IPC} \\
\cmidrule{3-5}
&&1 & 10 &50   \\
\midrule
\multirow{25}{*}{\rotatebox[origin=c]{90}{\bf Tiny-ImageNet}}
&DD~\cite{wang2018dataset} & - &- & -\\
&DD$^\dag$~\cite{wang2018dataset,deng2022remember} & - &-&-\\
&AddMem~\cite{deng2022remember} & -&-&-\\
&KIP~(ConvNet)~\cite{nguyen2021dataset} & -&- & -\\
&KIP~($\infty$-FC)~\cite{nguyen2020dataset}& - & - & - \\
&KIP~($\infty$-Conv)~\cite{nguyen2021dataset}& -&- & - \\
&FRePo~(ConvNet)~\cite{zhou2022dataset} & 15.4 $\pm$ 0.3 & 24.9 $\pm$ 0.2 & - \\
&FRePo~($\infty$-Conv)~\cite{zhou2022dataset}& 17.6 $\pm$ 0.2 & 25.3 $\pm$ 0.2 & - \\
&RFAD~(ConvNet)~\cite{loo2022efficient} & - & - & -\\
&RFAD~($\infty$-Conv)~\cite{loo2022efficient} & -&-& -\\
\cmidrule{2-5}
&DC~\cite{zhao2021dataset} & 5.3 $\pm$ 0.2 & 11.1 $\pm$ 0.3 & 11.2 $\pm$ 0.3 \\
&DSA~\cite{zhao2021datasetb} & 6.6 $\pm$ 0.2 & 16.3 $\pm$ 0.2 & 25.3 $\pm$ 0.2 \\
&IDC~\cite{kim2022dataset} &- & -& -\\
&EDD~\cite{zhang2022accelerating} & - & - & - \\
&DCC~\cite{lee2022contrastive} & - & - & -\\
&EGM~\cite{jiang2022delving} & - & - & - \\
&MTT~\cite{cazenavette2022dataset} & 8.8 $\pm$ 0.3 & 23.2 $\pm$ 0.2 & 28.2 $\pm$ 0.5 \\
&DDPP~\cite{li2022datasetb} &- & -& -\\
&HaBa~\cite{liu2022dataset} & -&-&- \\
&FTD~\cite{du2022minimizing} & 10.4 $\pm$ 0.3 & 24.5 $\pm$ 0.2 & - \\
&TESLA~\cite{cui2022scaling} & - & - & - \\
\cmidrule{2-5}
&DM~\cite{DBLP:journals/corr/abs-2110-04181} & 3.9 $\pm$ 0.2 & 13.5 $\pm$ 0.3 & 24.1 $\pm$ 0.3 \\
&IT-GAN~\cite{zhao2022synthesizing} & - &- &-\\
&KFS~\cite{lee2022dataset} & 22.7 $\pm$ 0.2 & 27.8 $\pm$ 0.2 & - \\
&CAFE~\cite{wang2022cafe} & - & - & -\\
\bottomrule
\end{tabular}
}
\label{tab:tinyimagenet}
\end{table}

\begin{table}[H]
\renewcommand{\arraystretch}{1.2}
\large
\centering
\caption{The performance results of all existing methods on ImageNet-1K. $\dag$ adopt momentum~\cite{deng2022remember}.}
\resizebox{\linewidth}{!}{
\begin{tabular}{cl c c cc }
\toprule
&& \multicolumn{4}{c}{\bf IPC} \\
\cmidrule{3-6}
&&1 &2 &10 &50   \\
\midrule
\multirow{25}{*}{\rotatebox[origin=c]{90}{\bf ImageNet-1K}}
&DD~\cite{wang2018dataset} & - &- & -&-\\
&DD$^\dag$~\cite{wang2018dataset,deng2022remember} & - &-&-&-\\
&AddMem~\cite{deng2022remember} & -&-&-&-\\
&KIP~(ConvNet)~\cite{nguyen2021dataset} & -&- & -&-\\
&KIP~($\infty$-FC)~\cite{nguyen2020dataset}& - & - & - &-\\
&KIP~($\infty$-Conv)~\cite{nguyen2021dataset}& -&- & - &-\\
&FRePo~(ConvNet)~\cite{zhou2022dataset} & 7.5 $\pm$ 0.3 & 9.7 $\pm$ 0.2 & -&- \\
&FRePo~($\infty$-Conv)~\cite{zhou2022dataset}& - & -& -&- \\
&RFAD~(ConvNet)~\cite{loo2022efficient} & - & - & -&-\\
&RFAD~($\infty$-Conv)~\cite{loo2022efficient} & -&-& -&-\\
\cmidrule{2-6}
&DC~\cite{zhao2021dataset} & - & -& - &-\\
&DSA~\cite{zhao2021datasetb} &-&-&-&- \\
&IDC~\cite{kim2022dataset} &- & -& -&-\\
&EDD~\cite{zhang2022accelerating} & - & - & - &-\\
&DCC~\cite{lee2022contrastive} & - & - & -&-\\
&EGM~\cite{jiang2022delving} & - & - & - &-\\
&MTT~\cite{cazenavette2022dataset} & -&-&-&-\\
&DDPP~\cite{li2022datasetb} &- & -& -&-\\
&HaBa~\cite{liu2022dataset} & -&-&- &-\\
&FTD~\cite{du2022minimizing} & - & - & -&- \\
&TESLA~\cite{cui2022scaling} & 7.7 $\pm$ 0.2 & 10.5 $\pm$ 0.3 & 17.8 $\pm$ 1.3 & 27.9 $\pm$ 1.2 \\
\cmidrule{2-6}
&DM~\cite{DBLP:journals/corr/abs-2110-04181} & 1.5 $\pm$ 0.1 & 1.7 $\pm$ 0.1  & - & -\\
&IT-GAN~\cite{zhao2022synthesizing} & - &- &- &-\\
&KFS~\cite{lee2022dataset} & -&-&- & - \\
&CAFE~\cite{wang2022cafe} & - & - & -&-\\
\bottomrule
\end{tabular}
}
\label{tab:imagenet-1k}
\end{table}


%





\ifCLASSOPTIONcaptionsoff
  \newpage
\fi

\end{document}